\documentclass[11pt]{article}
\usepackage{fullpage}
\usepackage{amsmath,amssymb,amsthm,mathtools}
\usepackage{tikz}
\usepackage{tocloft}
\usepackage{natbib}
\usepackage[colorlinks=true, linkcolor=blue, citecolor=blue]{hyperref}

\theoremstyle{plain}
\newtheorem{theorem}{Theorem}
\newtheorem{proposition}[theorem]{Proposition}
\newtheorem{lemma}[theorem]{Lemma}
\newtheorem{corollary}[theorem]{Corollary}
\theoremstyle{definition}
\newtheorem{definition}[theorem]{Definition}
\theoremstyle{remark}
\newtheorem{remark}[theorem]{Remark}

\theoremstyle{definition}
\newtheorem{observation}[theorem]{Observation}
\theoremstyle{remark}

\newcommand{\Pbb}{\mathbb{P}}
\newcommand{\Ebb}{\mathbb{E}}
\newcommand{\Rbb}{\mathbb{R}}
\newcommand{\MC}{\mathrm{MC}}
\newcommand{\SMC}{\mathrm{SMC}}
\newcommand{\SC}{\mathrm{SC}}
\newcommand{\Err}{\mathrm{Err}}
\newcommand{\KL}{D_{\mathrm{KL}}}
\newcommand{\distH}{\mathrm{dist}_H}

\newcommand{\Unif}{\mathrm{Unif}}
\newcommand{\Ber}{\mathrm{Ber}}
\newcommand{\ind}[1]{\mathbf{1}\!\left\{#1\right\}}
\newcommand{\ThetaCode}{\Theta_m}
\newcommand{\code}[1]{c(#1)}

\newcommand{\E}{\mathbb{E}}
\newcommand{\R}{\mathbb{R}}

\title{The Sample Complexity of Multicalibration}
 \author{Natalie Collina \qquad  Jiuyao Lu \qquad Georgy Noarov \qquad Aaron Roth\\{University of Pennsylvania}}
\date{}

\begin{document}
\maketitle

\begin{abstract}
We study the minimax sample complexity of multicalibration in the batch setting. A learner observes $n$ i.i.d.\ samples from an unknown distribution and must output a (possibly randomized) predictor whose population multicalibration error, measured by Expected Calibration Error (ECE), is at most $\varepsilon$ with respect to a given family of groups. For every fixed $\kappa > 0$, in the regime $|G|\le \varepsilon^{-\kappa}$, we prove that $\widetilde{\Theta}(\varepsilon^{-3})$ samples are necessary and sufficient, up to polylogarithmic factors. The lower bound holds even for randomized predictors, and the upper bound is realized by a randomized predictor obtained via an online-to-batch reduction. This separates the sample complexity of multicalibration from that of marginal calibration, which scales as $\widetilde{\Theta}(\varepsilon^{-2})$, and shows that mean-ECE multicalibration is as difficult in the batch setting as it is in the online setting, in contrast to marginal calibration which is strictly more difficult in the online setting. In contrast we observe that for $\kappa = 0$, the sample complexity of multicalibration remains $\widetilde{\Theta}(\varepsilon^{-2})$ exhibiting a sharp threshold phenomenon.

More generally,  we establish matching upper and lower bounds, up to polylogarithmic factors, for a weighted $L_p$ multicalibration metric for all $1 \le p \le 2$, with optimal exponent $3/p$. We also extend the lower-bound template to a regular class of elicitable properties, and combine it with the online upper bounds of \cite{hu2025efficient} to obtain matching bounds for calibrating properties including expectiles and bounded-density quantiles.
\end{abstract}

\clearpage
\pagenumbering{gobble}
{\footnotesize
\tableofcontents
}
\thispagestyle{empty}
\clearpage
\pagestyle{plain}
\pagenumbering{arabic}
\setcounter{page}{1}

\section{Introduction}

\paragraph{Calibration and multicalibration.}
A predictor is \emph{calibrated} if, conditional on the value it predicts, the predicted value equals the expected outcome \citep{dawid1982well}. The standard quantitative measure of mis-calibration is the \emph{Expected Calibration Error (ECE)}, which sums the magnitude of the prediction-conditional bias across all prediction values:
\[
\mathrm{ECE} \;=\; \sum_{v} \bigl|\E[v - Y \mid \text{prediction} = v]\cdot \Pbb[\text{prediction}=v]\bigr|.
\]
\emph{Multicalibration}, introduced by \cite{hebert2018multicalibration}, strengthens calibration by requiring it to hold simultaneously on every subpopulation represented by a collection of \emph{group functions} $G$.  The \emph{Expected Multicalibration Error} is the maximum, over groups, of the group-weighted ECE.

Since its introduction, multicalibration and closely related notions have found a wide range of applications, from learning predictors that are simultaneously optimal for many loss functions --- \emph{omniprediction} \citep{gopalan2021omnipredictors,gopalan2023loss,okoroafor2025near} --- to strengthening complexity-theoretic constructions \citep{casacuberta2024complexity,dwork2025supersimulators}, to low complexity algorithms for distributed information aggregation \citep{collina2025tractable,collina2026collaborative}. These applications motivate a thorough understanding of the statistical requirements for achieving multicalibration.

\paragraph{Batch multicalibration.}
In the \emph{batch} or \emph{statistical} setting, a learner is given a finite family $G$ of binary groups together with $n$ i.i.d.\ samples from an unknown distribution $P$ over context--outcome pairs $(X,Y)$, and must output a (possibly randomized) predictor $Q$ whose \emph{population} multicalibration error with respect to $G$ is at most $\varepsilon$, with constant probability. 

The population multicalibration error is defined as follows (see Section~\ref{sec:preliminaries} for the precise definitions). Given a finitely supported randomized predictor $Q$ that, upon observing context $X=x$, draws its prediction from a distribution $Q_x$ supported on a finite set $V(Q) \subset [0,1]$, the signed bias on group $g$ at prediction value $v$ is $B_P(Q;v,g) = \E[g(X) \cdot Q_X(\{v\}) \cdot (v-Y)]$. The Expected Calibration Error on group $g$ sums the absolute biases across prediction values, and the Expected Multicalibration Error takes the maximum over groups:
\[
\MC_P(Q;G) \;:=\; \max_{g \in G} \sum_{v \in V(Q)} |B_P(Q;v,g)|.
\]

Multicalibration asks for substantially more  than plain (\emph{marginal}) calibration. If one only asks for calibration with no group structure, then a constant predictor that always outputs an estimate of the marginal mean $\E[Y]$ is calibrated, so the sample complexity is just $\Theta(\varepsilon^{-2})$, exactly as in mean estimation. In general, multicalibration can still permit substantial coarsening of the regression function. The real issue is to understand when a given group family rules out such coarsening and instead forces the learner to resolve a much finer estimate of $\E[Y\mid X]$ itself. Our lower bound gives a family with exactly this behavior.

\paragraph{Prior work on sample complexity.}
Despite almost a decade of algorithmic work on multicalibration, the optimal sample complexity --- as a function of the target error $\varepsilon$ --- has remained an open question. In all of the upper bounds discussed below, the group family $G$ may be an arbitrary finite class whose size is polynomial in the sample size, and its contribution to the error guarantee is only logarithmic and is absorbed into the $\widetilde O(\cdot)$ notation. This makes the polynomial-budget regime a natural target for a sharp minimax characterization. Different algorithms and error metrics have produced a wide range of exponents when translated to the common metric of mean ECE used in this paper (see Table~\ref{tab:intro-rates} and Appendix~\ref{app:rw-rate-conversions} for the conversions). After conversion to  ECE, the original deterministic algorithm of \cite{hebert2018multicalibration} gives sample complexity $\widetilde O(\varepsilon^{-7})$. \cite{gupta2021online} gave the first online multicalibration algorithms together with an online-to-batch reduction, yielding a randomized predictor with sample complexity $\widetilde{O}(\varepsilon^{-4})$. \cite{lee2022online} and \cite{haghtalab2023unifying} later developed game theoretic treatments of several multicalibration notions, and gave randomized algorithms with the same $\widetilde{O}(\varepsilon^{-4})$ sample complexity. Most recently, the online multicalibration algorithm of \cite{noarov2025high}, when combined with a suitable online-to-batch reduction (which we give in this paper), achieves $\widetilde{O}(\varepsilon^{-3})$, which is the best known upper bound---see also \citep{ghugeimproved}, who  independently gave another algorithm with the same rate.

On the lower-bound side, there has been less work. The trivial lower bound from mean estimation gives $\Omega(\varepsilon^{-2})$. To our knowledge, the only improvement is due to \cite{gibbs-tibshirani-omnipredict}, who proved an $\Omega(\varepsilon^{-5/2})$ sample complexity lower bound for the weaker notion of calibrated multiaccuracy (c.f. \cite{casacuberta2025global}), and only for \emph{deterministic} predictors. No lower bound beyond the trivial one was known to hold against \emph{randomized} predictors, and so it was consistent with known results that randomized predictors might reduce the sample complexity down to $\varepsilon^{-2}$. Indeed, the best known upper bounds continue to be realized only via randomized predictors.

Calibration has also been studied for statistics beyond means. Early work treated particular properties such as moments and quantiles \citep{jung2021moment,gupta2021online,bastani2022practical,jung2023batch,deng2023happymap,garg2024oracle}. \cite{noarov2023scope} introduced the general framework of property multicalibration for continuous scalar elicitable properties and gave  online and batch algorithms under mild regularity conditions. Very recently, \cite{hu2025efficient} obtained substantially sharper online upper bounds for a family of weighted $L_p$ multicalibration metrics on elicitable properties: the case $p=1$ is the ECE metric defined above, while larger values of $p$ measure higher moments of the bucketwise calibration bias, weighted by the prediction mass of each bucket. Prior to our work, no non-trivial sample-complexity lower bounds were known for any non-mean property, or even for means against randomized predictors.

\paragraph{Our contributions.}
We resolve the minimax sample-complexity of multicalibration up to polylogarithmic factors, by proving lower bounds using explicit polylogarithmically sized group families that therefore match the best known upper bounds throughout every polynomial-size budget regime $|G|\le \varepsilon^{-\kappa}$ with fixed $\kappa>0$. Our results hold for the standard mean property, extend to $L_p$ multicalibration metrics for $1 \le p \le 2$, and further extend to a regular class of elicitable properties including expectiles and bounded-density quantiles. The lower bounds hold even for randomized predictors, and the upper bounds are realized by randomized predictors that result from online-to-batch reductions.

Our minimax statements study the regime in which the group family may grow polynomially with $1/\varepsilon$, i.e.\ $|G|\le M_\kappa(\varepsilon)=\lceil \varepsilon^{-\kappa}\rceil$ for a fixed $\kappa>0$. This is the natural scale on which to compare with the known upper bounds, whose dependence on the group family is only logarithmic in $|G|$ and therefore contributes only polylogarithmic factors in $1/\varepsilon$. Definition~\ref{def:sample-complexity} formalizes the corresponding minimax sample-complexity problem.
In the notation of Definition~\ref{def:sample-complexity}, we determine $\SC_{\mathrm{mean\text{-}ECE}}^{(\kappa)}(\varepsilon)$ and $\SC_{L_p}^{(\kappa)}(\varepsilon)$ up to polylogarithmic factors for every fixed $\kappa>0$. For regular elicitable properties, we prove an abstract lower-bound template with the same polylogarithmic group-size dependence and combine it with matching upper bounds over the distribution classes covered by the online hypothesis in \cite{hu2025efficient}, yielding matching exponents for expectiles and bounded-density quantiles for every fixed $\kappa>0$. To summarize, we ask:
\begin{center}
\emph{What is the minimax sample complexity of multicalibration?}
\end{center}
The answer we give: $\widetilde{\Theta}(\varepsilon^{-3})$ for the ECE metric, and $\widetilde{\Theta}(\varepsilon^{-3/p})$ for the $L_p$ metric. This in particular separates the sample complexity of multicalibration from the sample complexity of plain \emph{marginal} calibration, whose batch sample complexity is $\widetilde{\Theta}(\varepsilon^{-2})$. Together with the tight online lower bound of \cite{collina2026optimal}, our result shows that the minimax exponents for mean-ECE multicalibration agree in the batch and adversarial online settings. This is in sharp contrast with ordinary marginal calibration: while the batch rate is $\widetilde{\Theta}(\varepsilon^{-2})$, \cite{qiao2021stronger,dagan2025breaking} showed that this rate cannot be obtained in the online setting.

\begin{table}[t]
\centering
\footnotesize
\begin{tabular}{|p{4cm}|p{7.5cm}|c|c|}
\hline
Paper & Guarantee & Rate in  ECE & Randomized \\
\hline
\multicolumn{4}{|c|}{\textbf{Lower Bounds}} \\
\hline
\cite{gibbs-tibshirani-omnipredict} & Batch lower bound for calibrated multiaccuracy / ECE in mean regression (for deterministic predictors) & $\Omega(\varepsilon^{-5/2})$ &  \\
\hline
\textbf{This paper} & Batch lower bound for multicalibration in mean regression and for more general properties & $\textbf{\boldmath$\widetilde \Omega(\varepsilon^{-3})$}$ & $\checkmark$ \\
\hline
\multicolumn{4}{|c|}{\textbf{Upper Bounds}} \\
\hline
\cite{hebert2018multicalibration} & Batch $(G,\alpha)$-multicalibration with discarded mass and a minimum-group-mass assumption & $\widetilde O(\varepsilon^{-7})$ &  \\
\hline
\cite{gupta2021online} & Online $(\alpha,n)$-mean multicalibration on $n$ buckets, plus online-to-batch & $\widetilde O(\varepsilon^{-4})$ & $\checkmark$ \\
\hline
\cite{haghtalab2023unifying} & Batch $(G,\varepsilon,\lambda)$ bucketed multicalibration in $L_\infty$ & $\widetilde O(\varepsilon^{-4})$ & $\checkmark$ \\
\hline
\cite{gopalan2022low} & Full multicalibration with interval basis, specialized back to mean ECE & $\widetilde O(\varepsilon^{-8})$ &  \\
\hline
\cite{globus2023boosting} & Weighted $L_2$-multicalibration from $L_2$-boosting & $\widetilde O(\varepsilon^{-10})$ &  \\
\hline
\cite{noarov2025high} & Online bucketed mean multicalibration, followed by rounding and \textbf{Proposition~\ref{prop:mean-ece-otb}} & $\textbf{\boldmath$\widetilde O(\varepsilon^{-3})$}$ & $\checkmark$ \\
\hline
\end{tabular}
\caption{Previously known and newly derived sample-complexity rates for (mean) multicalibration error as measured with ECE. For upper bounds, ``Randomized'' means the algorithm produces a probabilistic predictor. For lower bounds, ``Randomized'' means that the lower bound holds even for probabilistic predictors. Upper bounds hold for arbitrary group families whose cardinality can grow polynomially with the sample size; the calibration error bounds all grow logarithmically with the group family cardinality, which increases the sample complexity by only a $\mathrm{polylog}(1/\varepsilon)$ factor absorbed in the $\widetilde{O}(\cdot)$ notation. Some upper bounds from prior work were stated under different error metrics; see Appendix~\ref{app:rw-rate-conversions} for conversions.}
\label{tab:intro-rates}
\end{table}

\subsection{Our Results}

We now state our results more precisely.

\begin{enumerate}
    \item \textbf{Lower bound for mean ECE multicalibration (special case of Theorem~\ref{thm:lower-regular}).}
    For every sufficiently small $\varepsilon$, there exists a hard instance on a domain $[m]$ with
    \[
    m=\Theta\!\left(\frac{1/\varepsilon}{\log(1/\varepsilon)}\right)
    \]
    together with a family of only $O(\text{polylog}(1/\varepsilon))$ binary groups and a family of Bernoulli distributions, such that any possibly randomized learner achieving population multicalibration error $\varepsilon$ with probability at least $2/3$ requires
    \[
    n = \widetilde{\Omega}(\varepsilon^{-3})
    \]
    samples. The hard instance is a one-dimensional Bernoulli regression with a monotone mean function on $[m]$, showing that the optimal rate already arises on extremely simple instances.
    
    \item \textbf{Matching upper bound for mean ECE (Theorem~\ref{cor:mean-ece-sharper}).}
    We record the straightforward online-to-batch conversion of the online multicalibration algorithm of \cite{noarov2025high} into a randomized batch predictor that achieves population multicalibration error $\varepsilon$ using only
    \[
    n = \widetilde{O}(\varepsilon^{-3}\log|G|)
    \]
    samples. Together with the lower bound, this establishes $\widetilde{\Theta}(\varepsilon^{-3})$ as the optimal sample complexity for mean ECE multicalibration, up to polylogarithmic factors, for every fixed polynomial group-size budget $|G|\le \varepsilon^{-\kappa}$ with $\kappa>0$.
    
    \item \textbf{$L_p$ multicalibration for $1 \le p \le 2$ (Theorem~\ref{thm:lower-regular} and Theorem~\ref{thm:upper-main}).}
    We extend both bounds to the $L_p$ multicalibration metric (cf.\ \cite{hu2025efficient}). For every fixed $\kappa>0$, the optimal sample complexity is $\widetilde{\Theta}(\varepsilon^{-3/p})$ for every $p \in [1,2]$. The lower bound follows from a H\"older comparison relating $L_p$ error back to ECE, applied to the same compressed hard family. The upper bound combines the online $L_p$ algorithm in \cite{hu2025efficient} with an online-to-batch reduction.
    
    \item \textbf{Regular elicitable properties (Theorem~\ref{thm:lower-regular} and Theorem~\ref{thm:upper-main}).}
    We introduce a regular class of elicitable properties, encompassing means, expectiles, and bounded-density quantiles among other examples. For this class, the same lower-bound template yields $\widetilde{\Omega}(\varepsilon^{-3/p})$ for the property-specific weighted $L_p$ multicalibration metric. When combined with the online upper-bound framework of \cite{hu2025efficient}, this gives matching batch exponents of $3/p$ for expectiles and bounded-density quantiles, up to polylogarithmic factors, for all $p \in [1, 2]$ and all fixed $\kappa>0$.
\end{enumerate}

We remark on the exponent $\kappa$ that determines how quickly the group family is allowed to grow as a function of $\varepsilon$. The lower-bound construction uses only polylogarithmically many groups in the hard-instance parameter $m$; after choosing $m$ as a function of $\varepsilon$, this becomes $|G|=\mathrm{polylog}(1/\varepsilon)$. As a result, the sharp lower bounds $\widetilde{\Omega}(\varepsilon^{-3})$ for mean ECE and $\widetilde{\Omega}(\varepsilon^{-3/p})$ for weighted $L_p$ already fit every polynomial budget $|G|\le \varepsilon^{-\kappa}$ with fixed $\kappa>0$. At $\kappa=0$ (corresponding to constant-size group families that do not grow with $1/\varepsilon$), no separation is possible for binary groups: if $|G|$ is fixed, then the group-membership patterns partition the domain into at most $2^{|G|}$ cells, and predicting the empirical mean on each cell gives sample complexity $\widetilde O(2^{|G|}\varepsilon^{-2})=O(\varepsilon^{-2})$. The remaining question is therefore the sharp joint dependence on $\varepsilon$ and $|G|$, especially in the regime in which $|G|$ grows with $\epsilon$ but only polylogarithmically.

\subsection{Lower Bound Proof Overview}

Our lower bounds are fully stated and proved in Section~\ref{sec:lower-bounds}, with the general result (which applies to all regular properties and to $L_p$ metrics) is recorded in Theorem~\ref{thm:lower-regular}. Here, we now give a proof overview in the language of our standard ($L_1$, mean) multicalibration sample complexity lower bound, in order to elide technicalities associated with the general statement and proof; still, we will address the extension to other $L_p$ norms and other distributional properties at the end of this proof sketch. (Since the complete proof in Section~\ref{sec:lower-bounds} proves the lower bound for general regular distributional properties, many of the intermediate lemmas which we reference here are proved in greater generality than we describe here.)

\paragraph{The Hard Instance.} To establish our $\Tilde{\Omega}(\epsilon^{-3})$ lower bound, we construct a sequence of hard instances parameterized by $m = \Tilde{\Theta}(1/\epsilon)$. For each $m$, the hard instance consists of:
\begin{enumerate}
    \item A hard binary group family $G_m$ (Definition~\ref{def:hard-groups}); 
    \item A hard family of data distributions $\mathcal{H}_m$ (Definition~\ref{def:hard-family}).
\end{enumerate}
At a high level, we construct $\mathcal{H}_m$ as a parametric family $(P_\theta)_{\theta \in \Theta_m}$ of distributions $P_\theta$ over $\mathcal{X} \times \mathcal{Y}$, with the parameter space $\Theta_m$ satisfying two desiderata: being (1) very large, and (2) pairwise well-separated. These two properties are clearly in tension, but, borrowing from a coding theory perspective, it turns out that they are possible to satisfy at the same time: such a collection $\Theta_m$ can be taken to be any code satisfying a \emph{packing property}.

The construction of the individual joint distributions $P_\theta$ must be quite specific, and their conditional label distributions crucially satisfy a certain monotonicity property, which is formalized via what we call \emph{staircase maps} $t_\theta$ ($\theta \in \Theta_m$). We will describe this in more detail shortly. 

Meanwhile, the hard group family $G_m$ is constructed so as, at a high level, to: (1) be small: $|G_m| = O(\text{polylog} (m))$, and (2) give rise to a well-behaved (in the sense of boundedness/slow growth of the coefficients) approximate representation of all possible sign threshold functions on the domain $[m]$. The significance of this latter property will be described shortly, but we note that it is tightly interlinked with the already mentioned monotonicity property of staircase maps. Once again, we see that both of these properties are in direct tension --- as one must represent the threshold function family by a well-behaved approximate basis much smaller than its cardinality --- yet again, a coding theory and a discrepancy theory insight enables both of these properties to hold at once. Specifically, the groups in $G_m$ will be constructed in a \emph{dyadic representation} manner, which decomposes every prefix interval into a disjoint union of dyadic intervals at different scales and then represents each dyadic block by a short, nearly orthogonal sign vector obtained from a low-correlation code.
(This group construction is inspired by \cite{gopalan2024omnipredictors}, which studies omniprediction for regression and builds interval approximations from low-rank factorizations of the identity matrix via results based on Johnson-Lindenstrauss lemma; but differs in the technical details.)

The existence of both types of codes, with the properties required for constructing $\mathcal{H}_m$ and $G_m$, follows in a simple way from an elementary probabilistic method-based lemma, allowing us to concisely present\footnote{We note that the results in that section are folklore (e.g.\ the packing argument is of Gilbert-Varshamov type), and our goal is to present them in a brief and unifying manner for the benefit of the reader, underscoring the remarkable point that our entire hard instance is constructed with the help of a single basic coding theory primitive.} the few coding theory primitives that we need in Section~\ref{subsec:coding-theory}.

\paragraph{Proof structure.}
Now, we describe the steps that our lower bound proof goes through. At a high level, we start with any learner that, with high probability, enforces approximate $G_m$-multicalibration on any joint distribution $P_\theta \in \mathcal{H}_m$. We then show, crucially relying on the properties of $G_m$ and $\mathcal{H}_m$, that such a multicalibrated learner must in fact be a \emph{good predictor} on $\mathcal{H}_m$. Then, once again relying on the properties of the hard distribution family $\mathcal{H}_m$, we show that a precise-enough predictor on this hard family is in fact an \emph{exact $\theta$-decoder}, i.e.\ must be able to exactly (without approximation error) determine the ground truth parameter $\theta \in \Theta_m$ using its sample from $P_\theta$. And finally, using that the parameter space $\Theta_m$ is very large, we show via an application of Fano's inequality that an exact $\theta$-decoder must have a high sample complexity, which evaluates to $\Omega(m^3) = \Tilde{\Omega}(\epsilon^{-3})$. Via the just-laid-out chain of equivalencies, this exact-decoding lower bound thus also serves as a multicalibration lower bound, just as we had set out to show. 

In what follows, we now describe each of the proof steps in more detail, referencing the relevant intermediate lemmas established in Section~\ref{sec:lower-bounds}.

\paragraph{Step 1: A $G_m$-multicalibrated predictor is a low-error predictor:}
In Proposition~\ref{prop:anticoarsen}, we show that for any distribution $P_\theta$ as we construct it, a learner's prediction error is bounded by its $G_m$-multicalibration error with an additional logarithmic factor:
\[
\text{Learner's prediction error} \leq O(\log m) \cdot (\text{Learner's } G_m\text{-multicalibration error}).
\]
To understand the mechanics of this step, we must now discuss how each $P_\theta$ is constructed. For means, $P_\theta$ is the following joint distribution. Contexts are distributed uniformly over the finite domain $[m]$, and the conditional distributions are chosen as $Y|X=i \sim Ber(t_\theta(i))$ for all $i \in [m]$. Here, the key object is $t_\theta$, indexed by $\theta$, which we call a \emph{staircase map}. For the illustration of this map, see Figure~\ref{fig:staircase}. This map is crucially defined to (1) be monotonically increasing, and (2) ``hide'' the bits in its defining codeword $\theta$ across its domain.

Now, recall that the group family $G_m$ is defined to approximately represent all the $m$ signed threshold functions on the domain $[m]$. The key point in proving this step's assertion is to note
that for the staircase $t_\theta$, due to its monotonicity, \emph{every sign pattern of $v-t_\theta(i)$ (where $v$ is any prediction) is always equivalent to one of the threshold functions in the
domain $[m]$}. Since our groups $G_m$ approximate every such threshold sign with only
polylog($m$) signed test functions, 
we may thus \emph{convert the learner's prediction error (which depends on the error patterns $v-t_\theta(i)$) into a decomposition in terms of the group functions in $G_m$ --- and thus (up to a loss in the boundedness parameter of the coefficients, which is only a $\log(m)$ factor) into the learner's multicalibration error.}
To summarize, this is a delicate
structural step: one must simultaneously approximate all thresholds while keeping the group
family small enough to obtain the desired minimax lower bound.

\paragraph{Step 2: A low-error predictor is an exact $\theta$-decoder:} Having demonstrated that small-error multicalibration forces small prediction error, we now carry this line of reasoning forward, and show that once a small-enough prediction error is achieved, it becomes possible for the learner to exactly determine the ground truth parameter value $\theta$. This is shown in Proposition~\ref{prop:decode}. 

The key property of our hard instance leveraged in this result is that any two distinct \emph{staircase maps} $t_\theta, t_{\theta'}$ for $\theta, \theta' \in \Theta_m$ (where $\theta \neq \theta'$) must be well-separated in $L_1$ distance. This well-separation property (Lemma~\ref{lem:staircase-separation}) more precisely states that any such pair of staircase maps $t_\theta, t_{\theta'}$ are at least $\Omega(1/m)$ apart in $L_1$-distance, and follows from the packing property of the code $\Theta_m$ via a direct connection via Hamming distance.

With this well-separation property in hand, it is then clear that by driving down the prediction error below $\Tilde{O}(1/m)$, and then taking the nearest neighbor of the best estimate $\hat{\theta}$ within the code $\Theta_m$, the ground-truth $\theta$ can be recovered exactly.

\paragraph{Step 3: An exact $\theta$-decoder has high sample complexity:} Now consider any learner who with high probability is able to exactly determine, from an i.i.d.\ sample of size $n$ from $P_\theta$, the true distribution parameter $\theta \in \Theta_m$. At this point, we can invoke Fano's inequality to demonstrate that this determination requires a large sample size. Quantitatively, Lemma~\ref{lem:fano-close} observes that (1) the parameter space $\Theta_m$ is very large (exponential in $m$), and (2) by a regularity property of distribution means (more explicitly, by a standard quadratic upper bound on KL divergence), it holds that any two disributions in $\mathcal{H}_m$ are $O(1/m^2)$-close in KL divergence --- and that by Fano's inequality, these two properties imply the sought sample complexity bound 
\[
n \geq \Omega \left(\frac{\log |\Theta_m|}{1/m^2} \right) = \Omega \left(\frac{m}{1/m^2} \right) = \Omega(m^3) = \Tilde{\Omega}(\epsilon^{-3}).
\]

\paragraph{Generalization to regular properties and $L_p$ metrics:} The extension to an optimal lower bound for other $L_p$ multicalibration metrics turns out to be immediate for $p \in [1, 2]$. For this range of $p$, simply applying Holder's inequality with our ECE ($L_1$) lower bound extends it to the lower sample complexity bound $\Tilde{\Omega}(\epsilon^{-3/p})$, and our upper bounding section later confirms this to be optimal via an online-to-batch reduction from the method of \cite{hu2025efficient}, whose bounds (after conversion) turn out to match this. 

However, as we also briefly mention in the future work (conclusion) section, proving tight lower bounds for $L_p$ multicalibration for $p \in (2, \infty]$ --- which is lossy with respect to the Holder conversion that we just described --- does not appear to easily follow from our $L_1$ lower bound construction. We therefore leave tight $L_p$-multicalibration bounds for $p > 2$ to future work, and note that this future direction underscores the nuanced nature of obtaining multicalibration lower bounds.

Next, generalizing our lower bounds to multicalibration of distributional properties $\Gamma$ beyond means requires care. Instead of the mean residuals $v-\mu(i)$, the general proof deals with the expected identification function $M_\Gamma(v,t)$, and is not automatic without further assumptions on $\Gamma$. We define ``regular'' properties $\Gamma$ as those that satisfy a certain collection of assumptions (see Definition~\ref{def:regular}).

In a nutshell, first, the identification function must be quantitatively non-flat. 
Secondly --- and critically --- our $\Gamma$-specific hard distribution family must satisfy a Fano-type requirement that all constituent pairs of distributions $P_\theta, P_{\theta'}$ be quadratically close in KL distance. This ensures that exact $\theta$-decoders of $P_\theta$ remain subject to a sample lower bound via Fano's inequality. 

We note that in terms of adjustments to the above-explained hard instance, only the hard distribution family needs to vary with the property $\Gamma$; meanwhile, the hard group family $G_m$ remains $\Gamma$-agnostic. Finally, our regularity assumptions on $\Gamma$ are not very restrictive, and we illustrate this by proving that --- aside from means --- \emph{quantiles} are also regular properties, subject to natural mild Lipschitzness, as are \emph{expectiles}; see the Appendix.

\subsection{Upper Bound Proof Overview}

The upper bound proofs are contained in Section~\ref{sec:upper-bounds}. The blueprint is somewhat standard, and proceeds by constructing and proving appropriate online-to-batch reduction statements. However, one does encounter several technical subtleties in comparison to familiar online-to-batch reductions for no-regret learning, and so we discuss the outline here. First, we describe the canonical mean-ECE case, whose upper bound guarantee is stated in Theorem~\ref{cor:mean-ece-sharper} and proved in Section~\ref{subsec:upper-mean-sharp}. Then, we turn to the general upper bound for $L_p$ property multicalibration, which is formally stated in Theorem~\ref{thm:upper-main} and proved in Sections~\ref{subsec:upper-otb} and~\ref{subsec:upper-plug-in}.

\paragraph{Upper bound for mean ECE multicalibration.}
For the ECE metric, the online-to-batch conversion is relatively straightforward. An online algorithm processes a stream of i.i.d.\ samples $(X_1,Y_1), \ldots, (X_T,Y_T)$ and, at round $t$, outputs a distribution-valued rule $q_t:\mathcal{X}\to\Delta(\Lambda)$ on a finite grid $\Lambda=(v_k)_{k=1}^K$. The realized prediction at round $t$ is then sampled from $q_t(X_t)$. We convert the transcript into a batch predictor by averaging the roundwise grid distributions: $Q_S := \frac{1}{T}\sum_{t=1}^T q_t$.

Let $\widehat{\MC}_T(S;G,\Lambda)$ denote the normalized empirical multicalibration error of the transcript on the grid $\Lambda$, and for each group $g$ and bucket $k$ let $\widehat B_k(S;g)$ and $B_k(S;g)$ denote the corresponding empirical and population bucket biases. Then the population bias of $Q_S$ on each group--bucket pair differs from the corresponding empirical bias by a martingale difference:
$B_k(S;g) - \widehat{B}_k(S;g) \;=\; \frac{1}{T}\sum_{t=1}^T M_t,$
where each $M_t$ is bounded and has conditional mean zero given the history. An Azuma--Hoeffding bound with a union over groups and sign patterns yields
$\E\!\left[\MC_P(Q_S;G)\right] \;\le\; \E\!\left[\widehat{\MC}_T(S;G,\Lambda)\right] + O\!\left(\sqrt{\frac{K + \log|G|}{T}}\right).$

Now, to instantiate the batch predictor, we utilize the online algorithm of~\cite{noarov2025high}. As they prove, it achieves cumulative bucketed multicalibration error $\widetilde{O}(T/K + \sqrt{TK})$ with $K$ prediction buckets. After normalizing by $T$ and rounding buckets to their centers, the empirical ECE of their algorithm is
$\widehat{\MC}_T(S;G,\Lambda) \;\le\; \widetilde{O}\!\left(\frac{1}{K} + \sqrt{\frac{K}{T}}\right).$
Setting $K = \widetilde{\Theta}(T^{1/3})$ makes both online terms $\widetilde{O}(T^{-1/3})$, and the online-to-batch transfer term is on the same order, $\widetilde{O}(T^{-1/3})$, giving the bound $\E[\MC_P(Q_S;G)] \le \widetilde{O}(T^{-1/3})$ and thus sample complexity $\widetilde{O}(\varepsilon^{-3}\log|G|)$ after Markov's inequality. We highlight that it is due to this same-order transfer loss that our batch lower bound does not recover the incomparable online multicalibration lower bound of~\cite{collina2026optimal}.

\paragraph{General ($L_p$, property) upper bound.}
In the general case the reduction is more delicate. The main reason is that the $L_p$ error metric involves the ratio $|B_k|^p / \pi_k^{p-1}$, where $\pi_k$ is the population mass of bucket $k$. This quantity is both nonlinear, and can blow up in ``light'' buckets where $\pi_k$ is small. This requires a more nuanced online-to-batch reduction, which must in particular use a more complex concentration bound than the familiar Azuma-Hoeffding.

To address this, we follow a variance-adaptive Freedman strategy (with a dyadic peeling component) introduced by~\cite{hu2025efficient}, and handle the ratio structure of the metric via a two-part argument. Namely, we set a threshold $\tau = \widetilde{\Theta}(1/T)$ and separately consider light and heavy buckets.
\emph{(1) Light buckets} (population mass $\pi_k < \tau$) contribute at most $K\tau$ total $L_p$ error, by the trivial bound $|B_k|^p / \pi_k^{p-1} \le \pi_k$.
\emph{(2) Heavy buckets} ($\pi_k \ge \tau$) are handled with the aforementioned strengthened Freedman's inequality of~\cite{hu2025efficient}, a variance-adaptive concentration inequality; it is needed because while standard Freedman inequality controls a martingale sum in terms of its predictable quadratic variation, here the variance proxy $\pi_k$ is itself a random quantity (it depends on the online algorithm's choices). As in \cite{hu2025efficient}, the proof takes a union over dyadic scales of $\pi_k$, paying only a $\log T$ factor, and then uses $\pi_k \ge \tau$ on each scale to absorb the lower-order terms.

Having obtained the just-described online-to-batch transfer guarantee, we instantiate the batch predictor with the recent online $L_p$ multicalibration algorithm of \cite{hu2025efficient}. This algorithm targets online \emph{swap} multicalibration guarantees, which are in fact somewhat stronger than ordinary multicalibration (and we in fact present the above online-to-batch reduction in the language of the swap error metric). For $p \in [1, 2]$, the algorithm of~\cite{hu2025efficient} uses discretization $K = \widetilde{\Theta}(T^{1/3})$ and leads, after the batch conversion, to the general multicalibration upper bound of $\widetilde{O}(\varepsilon^{-3/p})$, which matches our general multicalibration lower bound.

\subsection{Additional Related work}
\paragraph{Uniform Convergence Bounds} A related but distinct line of work studies \emph{uniform convergence} of multicalibration error over a fixed predictor class. \cite{shabat2020uniform} give uniform convergence bounds for bucketed multicalibration error with logarithmic dependence on the size of a finite predictor class and graph-dimension dependence for infinite classes, together with a matching lower bound for that uniform-convergence problem. \cite{rosenberg2022exploration} show more generally that such multicalibration uniform-convergence guarantees can often be obtained by reparametrizing standard ERM sample-complexity bounds, and instantiate this viewpoint using VC- and Rademacher-complexity analyses for several model classes. These results do not establish the sample complexity of learning a predictor with low multicalibration error: they control the gap between empirical and population multicalibration error uniformly over a fixed class $H$, but they do not imply that $H$ contains any predictor with small population multicalibration error, which also requires that the class itself is sufficiently expressive. In particular, they are compatible with every predictor in $H$ having large multicalibration error, whereas our results characterize the sample complexity of producing a predictor whose population multicalibration error is small.

\paragraph{Other Multicalibration Lower Bounds} \cite{gibbs-tibshirani-omnipredict} prove the only previous lower bound that we are aware of on the sample complexity of batch multicalibration, beyond the $\Omega(\varepsilon^{-2})$ lower bound that follows from marginal calibration/mean estimation (in fact their lower bound is for the strictly weaker measure of calibrated multiaccuracy). In our notation, their result implies that for the problem of mean regression, learning a deterministic predictor with expected (mean) multicalibration error $\varepsilon$ requires at least 
$n\geq \Omega(\varepsilon^{-5/2})$ many samples. They also give an $n \leq O(\varepsilon^{-3})$ sample complexity upper bound for the weaker notion of calibrated multiaccuracy. In contrast, we identify the minimax optimal sample complexity for multicalibration: $\widetilde{\Theta}(\varepsilon^{-3})$, and our lower bound holds even for randomized predictors; similarly our analysis extends also to other multicalibration error metrics and to other elicitable properties beyond the mean. 

In a complementary direction, \cite{gopalan2024multiclass} study computationally efficient multi-class calibration when the label space has size $k$. Their lower bounds are mainly about the \emph{auditing} or recalibration problem of improving a given predictor while preserving squared loss, and show that more expressive multi-class notions such as canonical calibration and full smooth calibration require exponentially many samples in $k$ (and, for some notions, also face computational hardness). These results concern high-dimensional label spaces and a different recalibration task, and are therefore orthogonal to our minimax sample-complexity lower bounds for scalar multicalibration.

\cite{collina2026optimal} characterize the optimal rate for mean multicalibration in the online adversarial setting under the same ECE metric, proving a tight $\widetilde{\Theta}(T^{-1/3})$ normalized error rate. Their result and ours are incomparable: just as we characterize the optimal batch sample complexity, they characterize the optimal online rate, and neither theorem implies the other. In particular, our online-to-batch reductions incur an additive loss term of order $\widetilde \Omega(T^{-1/3})$, so our batch lower bounds do not lift to non-vacuous online lower bounds by contrapositive. 

Our high-level lower bound proof outline intersects with, but substantially differs from, that of \cite{collina2026optimal}. 
As a place of substantial difference, their hard distribution is tailored to the adversarial setting and encodes the mean as $\E[Y\mid X]=X$ exactly, whereas our batch lower bound uses a parametric monotone stochastic staircase family powered by certain (packing-type) binary codes. 
On the other hand, there is overlap insofar \emph{group family} component of the hard instance is concerned. Specifically, a major technical purpose of the group construction in both proofs is that it provides a succinct basis for approximating families of sign threshold functions on a finite domain; this aids with the high-level step of proving that low multicalibration error implies low prediction error (which they refer to as ``truthfulness''). However, the exact group construction we use here is different: \cite{collina2026optimal} use a subsampled Walsh basis technique, while we employ a coding theory approach coupled with a dyadic representation inspired by~\cite{gopalan2024omnipredictors}. (We note that their subsampled Walsh approach could also be used to define groups in the place of our coding approach, but we present a fully coding theory-based construction as a simple and consistent approach.)

\section{Preliminaries}
\label{sec:preliminaries}

We begin with the population notion used in the mean-calibration lower and upper bounds. In this section we first define multicalibration for the mean property, then introduces the analogous definitions for general elicitable properties.

Fix a context space $\mathcal{X}$, a distribution $P$ on $\mathcal{X}\times[0,1]$, and a family $G$ of binary groups $g:\mathcal{X}\to\{0,1\}$. Let $\mu(x):=\E[Y\mid X=x]$ denote the regression function of $P$.

\subsection{Mean multicalibration for finitely supported randomized predictors}

\begin{definition}[Finitely supported randomized predictor]
A finitely supported randomized predictor on $\mathcal{X}$ is a rule $Q=(Q_x)_{x\in\mathcal{X}}$ that assigns to each context $x\in\mathcal{X}$ a probability distribution $Q_x$ on $[0,1]$, and for which there exists a finite set $V(Q)\subseteq [0,1]$ such that $Q_x(V(Q))=1$ for all $x\in\mathcal{X}$.
\end{definition}

\noindent Equivalently, after observing $X=x$, the predictor outputs a random value $V\sim Q_x$, and only finitely many prediction values can occur overall. We work with this form throughout the main text because it lets us write calibration error as an explicit sum over prediction values. Restricting to a finite support is without loss of generality: Appendix Proposition~\ref{prop:quantize-arbitrary} shows that any arbitrary randomized predictor can be quantized to a finite grid while changing multicalibration and prediction error by at most an arbitrarily small additive term.

\begin{definition}[Population mean multicalibration] \label{def:population-mean-multicalibration}
For a signed weight function $w:\mathcal{X}\to[-1,1]$ and a prediction value $v\in V(Q)$, define the population signed bias by
\begin{align*}
B_P(Q;v,w)
&:= \E\!\left[w(X)\,Q_X(\{v\})\,(v-Y)\right] \\
&= \E\!\left[w(X)\,Q_X(\{v\})\,(v-\mu(X))\right].
\end{align*}

Define the corresponding signed-weight error by
\[
\Err_P(Q;w):=\sum_{v\in V(Q)} |B_P(Q;v,w)|.
\]
When $g:\mathcal{X}\to\{0,1\}$ is a binary group, this is its Expected Calibration Error (ECE):
\[
\Err_P(Q;g):=\sum_{v\in V(Q)} |B_P(Q;v,g)|.
\]
Finally define the Expected Multicalibration Error of $Q$ with respect to $G$ by
\[
\MC_P(Q;G):=\max_{g\in G}\Err_P(Q;g).
\]
\end{definition}

\begin{definition}[Prediction error]
\label{def:prediction-error}
The average absolute prediction error of $Q$ is defined as
\[
\Delta_P(Q)
:=
\E\!\left[\int |v-\mu(X)|\,Q_X(dv)\right]
=
\sum_{v\in V(Q)}
\E\!\left[Q_X(\{v\})\,|v-\mu(X)|\right],
\]
with the equality holding because $Q$ is finitely supported.
\end{definition}

\begin{remark}
A \emph{deterministic} predictor $f:\mathcal{X}\to[0,1]$ is the special case in which $Q_x=\delta_{f(x)}$ for every $x\in\mathcal{X}$.
In that case $V(Q)=f(\mathcal{X})$, and for every signed weight $w$ and every $v\in f(\mathcal{X})$,
\[
B_P(Q;v,w)
=
\E\!\left[w(X)\ind{f(X)=v}(v-Y)\right]
=
\E\!\left[w(X)\ind{f(X)=v}(v-\mu(X))\right].
\]
Accordingly,
\[
\Err_P(Q;w)
=
\sum_{v\in f(\mathcal{X})}
\left|
\E\!\left[w(X)\ind{f(X)=v}(v-Y)\right]
\right|
\qquad \text{and} \qquad
\Delta_P(Q)=\E\!\left[|f(X)-\mu(X)|\right].
\]
The randomized notation is thus a direct extension of the usual, deterministic, multicalibration. We will freely identify deterministic predictors with corresponding point-mass randomized predictors.
\end{remark}

\subsection{Learners and minimax sample complexity}

We now formalize the statistical problem studied in the paper. The learner is given the group family $G$ together with the sample, and we measure sample complexity in a minimax sense over all distributions and all group families of bounded size.

\begin{definition}[Learners given a group family]
\label{def:learner}
A learner is a sequence $\mathcal{A}=(A_n)_{n\ge 1}$
such that for every context space $\mathcal{X}$, every outcome space $\mathcal{Y}$, every finite family of binary groups $G\subseteq \{0,1\}^{\mathcal{X}},$ and every sample
$S\in (\mathcal{X}\times\mathcal{Y})^n,$
and every internal random seed $\zeta$, the output
$A_n(G,S,\zeta)$ 
is a finitely supported randomized predictor on $\mathcal{X}$. 
\end{definition}

\begin{definition}[Minimax sample complexity]
\label{def:sample-complexity}
Fix a nonnegative error functional $\mathcal{E}$ that assigns a value $\mathcal{E}(P,G,Q)\in[0,\infty)$ 
to each distribution $P$ on $\mathcal{X}\times\mathcal{Y}$, each finite family of binary groups
$G\subseteq \{0,1\}^{\mathcal{X}},$
and each finitely supported randomized predictor $Q$ on $\mathcal{X}$. For a learner $\mathcal{A}$ as in Definition~\ref{def:learner}, define its instance-wise sample complexity by
\[
n_{\mathcal{A}}(\varepsilon;P,G,\mathcal{E})
:=
\inf\Bigl\{
n\ge 1:
\Pbb_{S\sim P^n,\ \zeta}
\!\left(
\mathcal{E}\!\left(P,G,A_n(G,S,\zeta)\right)\le \varepsilon
\right)\ge \frac23
\Bigr\}.
\]
For a size budget
$M:(0,1)\to \mathbb{N},$
define the minimax sample complexity by
\[
\SC_{\mathcal{E}}(\varepsilon;M)
:=
\inf_{\mathcal{A}}
\sup_{\substack{\mathcal{X},\,P,\,G:\\ G\subseteq\{0,1\}^{\mathcal{X}},\ |G|\le M(\varepsilon)}}
n_{\mathcal{A}}(\varepsilon;P,G,\mathcal{E}).
\]
For every $\kappa>0$, define the polynomial-size budget $M_\kappa(\varepsilon):=\left\lceil \varepsilon^{-\kappa}\right\rceil$
and write
\[
\SC_{\mathcal{E}}^{(\kappa)}(\varepsilon):=\SC_{\mathcal{E}}(\varepsilon;M_\kappa).
\]
\end{definition}
\begin{remark}[Restricted minimax sample complexity]
\label{rem:restricted-sample-complexity}
If $\mathcal{P}$ is a class of admissible distributions on context--outcome spaces, we write
\[
\SC_{\mathcal{E}}^{\mathcal{P}}(\varepsilon;M)
:=
\inf_{\mathcal{A}}
\sup_{\substack{(\mathcal{X},\mathcal{Y},P)\in\mathcal{P},\\ G\subseteq\{0,1\}^{\mathcal{X}},\ |G|\le M(\varepsilon)}}
n_{\mathcal{A}}(\varepsilon;P,G,\mathcal{E})
\]
for the corresponding minimax sample complexity with the supremum restricted to $P\in\mathcal{P}$, and
\[
\SC_{\mathcal{E}}^{\mathcal{P},(\kappa)}(\varepsilon):=\SC_{\mathcal{E}}^{\mathcal{P}}(\varepsilon;M_\kappa).
\]
When $\mathcal{P}$ is the unrestricted class of all distributions, this reduces to Definition~\ref{def:sample-complexity}.
\end{remark}

Note that $\SC_{\mathcal{E}}^{(\kappa)}$ allows the cardinality of the group function class $|G|$ to grow at a polynomial rate with $1/\epsilon$. This is the regime that will allow us to prove bounds that tightly match the best known upper bounds' dependence on $\epsilon$ (up to polylogarithmic factors), as these upper bounds have only a logarithmic dependence on $|G|$. 
\begin{remark}
\label{rem:sample-complexity-shorthand}
For the mean ECE metric, we write
\[
\mathcal{E}_{\mathrm{mean\text{-}ECE}}(P,G,Q):=\MC_P(Q;G),
\qquad
\SC_{\mathrm{mean\text{-}ECE}}:=\SC_{\mathcal{E}_{\mathrm{mean\text{-}ECE}}}.
\]
We will introduce analogous shorthand for the weighted $L_p$ metrics and for elicitable properties when those notions are defined later in the paper.
\end{remark}

\subsection{Elicitable properties}

% Up to this point the paper has focused on the mean property. We now introduce definitions of elicitable properties, their regularity, with the goal of demonstrating that our lower and upper bounds extend to a regular class of one-parameter elicitable properties, and then instantiate the abstract result for expectiles and quantiles. The lower-bound argument uses only the regularity conditions introduced below. The matching batch upper bounds later in the section additionally invoke the standard Lipschitz hypothesis from the online results in \cite{hu2025efficient}. Proofs that mirror their analogues in the mean multicalibration case are deferred to Appendix~\ref{sec:regular-properties}.

% The section has four parts. We first define property multicalibration and the regular families that support the lower-bound construction. We then state the abstract lower and upper results: the lower bound requires a regular family, while the upper bound uses the online hypothesis in \cite{hu2025efficient}. After that we explain the abstract lower-bound template by building the hard family inside a regular parameter interval. Finally, we verify that expectiles and bounded-density quantiles satisfy the required assumptions.

% The learner and sample-complexity definitions from Section~\ref{sec:preliminaries} extend verbatim with the label space $[0,1]$ replaced by an arbitrary outcome space $\mathcal{Y}$. Throughout this section we use that generalized version.

We now extend the mean-specific notions above to general scalar elicitable properties. This section introduces the property-specific multicalibration metrics used later, the regularity condition underlying our lower-bound template, and the concrete examples needed in the sequel. The general lower bound for regular properties is proved in Section~\ref{sec:lower-bounds}, while the corresponding upper bounds under the online hypothesis in \cite{hu2025efficient} are proved in Section~\ref{sec:upper-bounds}. Proofs of the example regularity verifications are deferred to Appendix~\ref{app:regular-verifications}.

The learner and minimax sample-complexity definitions from Definitions \ref{def:learner} and \ref{def:sample-complexity}, together with the restricted version in Remark \ref{rem:restricted-sample-complexity}, extend verbatim when the outcome space $[0,1]$ is replaced by an arbitrary label space $\mathcal{Y}$. Throughout this section we use that extension.

Fix an elicitable property $\Gamma$ with a bounded identification function
\[
\mathsf{V}:[0,1]\times \mathcal{Y}\to [-1,1].
\]
\begin{definition}[Property multicalibration \citep{noarov2023scope,hu2025efficient}]
For a distribution $P$ on $\mathcal{X}\times \mathcal{Y}$, a finitely supported randomized predictor $Q$, a prediction value $v\in V(Q)$, and a signed weight $w:\mathcal{X}\to[-1,1]$, define the property-specific signed bias by
\[
B_P^\Gamma(Q;v,w)
:=
\E\!\left[w(X)\,Q_X(\{v\})\,\mathsf{V}(v,Y)\right].
\]
For a binary group $g:\mathcal{X}\to\{0,1\}$, define its ECE-type error by
\[
\Err_P^\Gamma(Q;g)
:=
\sum_{v\in V(Q)} \left|B_P^\Gamma(Q;v,g)\right|,
\qquad
\MC_P^\Gamma(Q;G)
:=
\max_{g\in G}\Err_P^\Gamma(Q;g).
\]
Also define the weighted $L_p$ version, for every $p\ge 1$, by
\[
\pi_P(Q;v):=\E\!\left[Q_X(\{v\})\right],
\quad
\Err_P^{\Gamma,(p)}(Q;g)
:= \mkern-30mu
\sum_{\substack{v\in V(Q)\\ \pi_P(Q;v)>0}} \mkern-20mu
\frac{|B_P^\Gamma(Q;v,g)|^p}{\pi_P(Q;v)^{p-1}},
\quad
\MC_P^{\Gamma,(p)}(Q;G)
:=
\max_{g\in G}\Err_P^{\Gamma,(p)}(Q;g).
\]
\end{definition}
When $\Gamma$ is the mean property and $\mathsf{V}(v,y)=v-y$, the case $p=1$ recovers Definition~\ref{def:population-mean-multicalibration}.
\begin{remark}
For an elicitable property $\Gamma$, we write
\[
\mathcal{E}_{\Gamma}(P,G,Q):=\MC_P^\Gamma(Q;G),
\qquad
\SC_{\Gamma}:=\SC_{\mathcal{E}_{\Gamma}}.
\]
For the corresponding weighted $L_p$ metric, we write
\[
\mathcal{E}_{\Gamma,p}(P,G,Q):=\MC_P^{\Gamma,(p)}(Q;G),
\qquad
\SC_{\Gamma,p}:=\SC_{\mathcal{E}_{\Gamma,p}}.
\]
\end{remark}

The next definition packages the ingredients needed for our lower bound construction. We need a one-parameter family of conditional label distributions whose parameter really is the target property, whose expected identification signal points toward the true parameter with linear strength, and whose nearby parameters are statistically hard to distinguish.

\begin{definition}[Regular property]
\label{def:regular}
Fix a bounded scalar property $\Gamma$ with identification function~$V$.

Let $I_0\subset(0,1)$ be a nondegenerate closed interval. A one-parameter family
$\{D_t:t\in I_0\}$ is a
\emph{regularity witness} for $\Gamma$ on $I_0$ if, for some constants
$c_\Gamma,C_{\Gamma,\mathrm{KL}}>0$, the following conditions hold:
\begin{enumerate}
    \item for every $t\in I_0$, $\Gamma(D_t)=t$;
    \item for every $t\in I_0$ and every $v\in[0,1]$,
    \[
    (v-t)M_\Gamma(v,t)\ge c_\Gamma(v-t)^2,
    \qquad \text{where } M_\Gamma(v,t):=\Ebb_{Y\sim D_t}[V(v,Y)];
    \]
    \item for every $t,t'\in I_0$,
    \[
    \KL(D_t\,\|\,D_{t'})\le C_{\Gamma,\mathrm{KL}}(t-t')^2.
    \]
\end{enumerate}
$\Gamma$ is called \emph{regular} if it admits a regularity witness on
some nondegenerate closed $I_0\subset(0,1)$.
\end{definition}

Whenever $\{D_t:t\in I_0\}$ is a regularity witness, condition~(2) implies that
$M_\Gamma(v,t)$ has the same sign as $v-t$ and that
\[
c_\Gamma|v-t|\le |M_\Gamma(v,t)|
\qquad (v\in[0,1],\ t\in I_0).
\]

\begin{remark}
Each part of Definition~\ref{def:regular} plays a specific role in the proof. Condition (1) says that the scalar parameter $t$ genuinely indexes the property value we are trying to predict. Condition (2) is the property analogue of the mean identity $\E[v-Y\mid t]=v-t$: it says that the expected identification signal has the correct sign and is not too flat, so small property multicalibration forces predictions to be close to the true parameter. Condition (3) is the information-theoretic input needed for Fano's inequality: it says that changing the hidden parameter by $\delta$ changes the conditional law by only $O(\delta^2)$ in KL. In the examples below, means, expectiles, and bounded-density quantiles all fit this template.
% We will also say that an elicitable property $\Gamma$ is \emph{regular} if it admits at least one regular family in the sense of Definition~\ref{def:regular}.
\end{remark}

Next, we ascertain that our regularity definition is benign enough to include some fundamental distributional properties. The regularity proofs for these properties are deferred to
Appendix~\ref{app:regular-verifications}.
\begin{theorem}[Examples of Regular Properties]
\label{thm:regular-examples}
The following properties are regular:
\begin{enumerate}
    \item The mean property, with identification function
$V_{\mathrm{mean}}(v,y)=v-y$, admits a Bernoulli regularity witness on every
closed interval $I_0\subset(0,1)$; on $I_0=[1/4,3/4]$ one may take
$c_\Gamma=1$ and $C_{\Gamma,\mathrm{KL}}=16/3$.
    \item For every expectile level $\tau\in(0,1)$, with identification function
\[
V_\tau(v,y):=|\tau-\ind{y\le v}|(v-y),
\]
the expectile property admits a regularity witness on the fixed interval
$[1/4,3/4]$, and for every distribution $\nu$ on $[0,1]$ the map
\[
v\mapsto \Ebb_{Y\sim \nu}[V_\tau(v,Y)]
\]
is $1$-Lipschitz on $[0,1]$.
    \item For every quantile level $q\in(0,1)$, with identification function
\[
V_q(v,y):=\ind{y\le v}-q,
\]
there exist a constant $\Lambda_q>0$, a closed interval $I_q\subset(0,1)$,
and a constant $C_q<\infty$ such that the truncated-exponential family
$\{D_\lambda:\lambda\in[-\Lambda_q,\Lambda_q]\}$, reparameterized by its
$q$-quantile, is a regularity witness on $I_q$, and every distribution in
that witness family has density bounded above by $C_q$.
\end{enumerate}
\end{theorem}

\begin{proof}[Proof of Theorem~\ref{thm:regular-examples}]
The mean claim is Proposition~\ref{prop:mean-regular-app}, the expectile claim is
Proposition~\ref{prop:expectile-regular-app}, and the quantile claim is
Proposition~\ref{prop:quantile-regular-app}.
\end{proof}

% \begin{definition}[Regular one-parameter family]
% \label{def:regular-family}
% Fix a nondegenerate interval
% \[
% I=[a,b]\subset (0,1).
% \]
% We call a family
% \[
% (\Gamma,\mathsf{V},\{D_t:t\in I\})
% \]
% regular if there are constants $c_\Gamma,C_{\Gamma,\mathrm{KL}}>0$ such that:
% \begin{enumerate}
%     \item $\Gamma(D_t)=t$ for every $t\in I$;
%     \item if
%     \[
%     M_\Gamma(v,t):=\E_{Y\sim D_t}\!\left[\mathsf{V}(v,Y)\right],
%     \]
%     then for every $v\in[0,1]$ and every $t\in I$,
%     \[
%     (v-t)\,M_\Gamma(v,t)\ge c_\Gamma (v-t)^2;
%     \]
%     \item for every $t,t'\in I$,
%     \[
%     \KL(D_t\,\|\,D_{t'})\le C_{\Gamma,\mathrm{KL}}(t-t')^2.
%     \]
% \end{enumerate}
% The second condition implies both that $M_\Gamma(v,t)$ has the same sign as $v-t$ and that
% \[
% c_\Gamma |v-t|\le |M_\Gamma(v,t)|
% \qquad
% (v\in[0,1],\ t\in I).
% \]
% \end{definition}

\section{Lower Bounds}
\label{sec:lower-bounds}

We now provide our general sample complexity lower bound for multicalibration. This bound holds not only for classical --- mean ECE --- multicalibration, but in fact for all regular properties $\Gamma$ and all $L_p$-multicalibration metrics for $p \in [1, 2]$. We now formally state our lower bound as the following Theorem~\ref{thm:lower-regular}, which we prove in the rest of this Section.

\begin{theorem}[General Multicalibration Lower Bound]
\label{thm:lower-regular}
Fix a regular property $\Gamma$, and fix $p\in[1,2]$. Then for small enough $\varepsilon>0$, there exists a parameter $m=\Theta\!\left(\frac{(1/\varepsilon)^{1/p}}{\log(1/\varepsilon)}\right)$ and:
\begin{itemize}
    \item a family $G_m$ of binary groups on $[m]$ with
    \[
    |G_m|=O\!\left(\log^4(1/\varepsilon)\right),
    \]
    \item and a finite family $\mathcal H_m^\Gamma$ of distributions on
    $[m]\times[0,1]$,
\end{itemize}
such that the following holds. 

Consider any possibly randomized learner whose input consists of: (1) the group family $G_m$,
and (2) $n$ i.i.d.\ samples from an unknown distribution
$P\in\mathcal H_m^\Gamma$, and whose output is a finitely supported randomized predictor
$Q$. Then, to ensure that $Q$ satisfies
\[
\Pbb\bigl(\MC_P^{\Gamma,(p)}(Q;G_m)\le \varepsilon\bigr)\ge \frac23
\qquad\text{for every }P\in\mathcal H_m^\Gamma,
\]
the learner's sample size must be at least\footnote{Here and in what follows, we may endow our asymptotic notation (such as $\Omega$) with a property subscript $\Gamma$ to highlight hidden dependencies on property-specific regularity constants.}:
\[
n
=\Omega_\Gamma\!\left(\frac{1}{\varepsilon^{3/p}\log^3(1/\varepsilon)}\right)
=\widetilde\Omega_\Gamma(\varepsilon^{-3/p}).
\]
\end{theorem}

\begin{corollary}[Minimax lower bound for regular properties]
\label{cor:lower-regular-minimax}
For every regular property $\Gamma$, every $p\in[1,2]$, and every fixed
$\kappa>0$,
\[
\SC_{\Gamma,p}^{(\kappa)}(\varepsilon)=\widetilde\Omega(\varepsilon^{-3/p}).
\]
\end{corollary}

\begin{proof}
Fix $\kappa>0$. By Theorem~\ref{thm:lower-regular}, the hard family at accuracy $\varepsilon$ uses a group family of size
$|G_m|=O\!\left(\log^4(1/\varepsilon)\right)=o(\varepsilon^{-\kappa}).$
Hence the same hard instances witness the lower bound on the minimax
sample complexity.
\end{proof}

\subsection{Coding theory primitives}
\label{subsec:coding-theory}

Our sample complexity lower bound will rely on some simple coding theory results. In this subsection, we give a self-contained overview of the requisite notions, and state and prove the existence of two (families of) codes that will power our lower-bounding instance construction. We point out that the results in the subsection are folklore and presented for a fully self-contained exposition highlighting the simplicity and unified nature of the coding primitive needed for our hard instance --- both on the group family side and on the distribution family side. For a general overview of coding theory, we refer i.a.\ to monographs such as \cite{roth2006introduction}. For a deeper look at the specific themes associated with the below lemmas, including i.a.\ balanced codes and Gilbert-Varshamov theory, we refer the reader i.a.\ to the seminal paper of Ta-Shma~\cite{ta2017explicit}; however, our arguments do not require any of the complexity of more advanced coding theory.

Recall that a \emph{binary code} of \emph{block length} $k$ is a subset $\mathcal C\subseteq\{0,1\}^k$. The distance between two codewords $c, c' \in \mathcal{C}$ will be taken to be the \emph{Hamming distance}
\[
\distH(c,c'):=\bigl|\{t\in[k]:c(t)\neq c'(t)\}\bigr|.
\]
Rather than working with binary codes, it is often more convenient to deal with sign vectors $z\in\{\pm1\}^k$. These two viewpoints are equivalent via the conversion 
$\code{z}(t):=\frac{1-z(t)}{2},$
such that the Hamming distance becomes 
$\distH(\code{z},\code{z'})
=
\frac{k-\langle z,z'\rangle}{2}.$
Thus, small inner products among sign vectors are the same as large Hamming separation among the associated binary codewords.

We are now ready to prove a simple yet key fact about the existence of certain codes. It is a standard probabilistic method argument via Hoeffding's inequality \cite[Theorem~2]{Hoeffding1963}. Via two distinct instantiations, which we will present next, this fact will underlie both facets of our lower bounding instance: the hard groups construction \emph{and} the hard distribution family construction.

\begin{lemma}[Low-correlation codes]
\label{lem:codewords}
Let $N\ge 2$, let $\rho\in(0,1)$, and let $k$ be an integer satisfying
\[
k\ge 8\rho^{-2}\log(2N).
\]
Then there exist vectors $z_0,\dots,z_{N-1}\in\{\pm1\}^k$ such that for all distinct $a,b\in\{0,\dots,N-1\}$,
\[
\bigl|\langle z_a,z_b\rangle\bigr|\le \rho k.
\]
Equivalently, the associated binary codewords $\code{z_0},\dots,\code{z_{N-1}}\in\{0,1\}^k$ satisfy
\[
\distH(\code{z_a},\code{z_b})
=
\frac{k-\langle z_a,z_b\rangle}{2}
\in
\left[\frac{1-\rho}{2}k,\frac{1+\rho}{2}k\right]
\qquad (a\neq b).
\]
\end{lemma}

\begin{proof}
Sample $Z_0,\dots,Z_{N-1}\in\{\pm1\}^k$ independently and uniformly at random. For fixed $a \neq b$,
\[
\langle Z_a,Z_b\rangle=\sum_{t=1}^k Z_a(t)Z_b(t)
\]
is a sum of $k$ independent Rademacher random variables. By Hoeffding's inequality,
\[
\Pbb\!\left(\bigl|\langle Z_a,Z_b\rangle\bigr|>\rho k\right)
\le 2\exp\!\left(-\frac{\rho^2k}{2}\right)
\le 2e^{-4\log(2N)}\le \frac{1}{8N^4}.
\]
There are fewer than $N^2/2$ unordered pairs $\{a,b\}$, so a union bound shows that with positive probability no pair violates the displayed correlation bound. Hence a desired realization exists.

The identity relating inner products and Hamming distance is
\[
\distH(\code{z_a},\code{z_b})
=
\sum_{t=1}^k\frac{1-z_a(t)z_b(t)}{2}
=
\frac{k-\langle z_a,z_b\rangle}{2},
\]
which immediately yields the equivalent binary-code formulation.
\end{proof}

Our first instantiation of the above lemma will provide a dyadic collection of codes that will be used in the definition of our hard group families (Section~\ref{subsec:hard-groups}). Namely, at all dyadic scales, it gives short sign signatures for dyadic blocks that are nearly orthogonal across distinct blocks. 

\begin{lemma}[Dyadic low-correlation code]
\label{lem:dyadic-signatures}
Let $m$ be a power of two, $L:=\log_2 m$, and
$\rho_m:=\frac{1}{8(1+L)}.$
For each scale $h\in\{0,1,\dots,L-1\}$, define
$n_h:=\frac{m}{2^h},
\;
k_h:=\left\lceil 8\rho_m^{-2}\log(2n_h)\right\rceil.$
Then there exist vectors
\[
z_0^{(h)},\dots,z_{n_h-1}^{(h)}\in\{\pm1\}^{k_h}
\text{ such that } 
\quad \bigl|\langle z_a^{(h)},z_b^{(h)}\rangle\bigr|\le \rho_m k_h 
\quad \text{ for all distinct $a,b\in\{0,\dots,n_h-1\}$.}
\]
\end{lemma}

\begin{proof}
Apply Lemma~\ref{lem:codewords} with $N=n_h$, $\rho=\rho_m$, and $k=k_h$ separately for each scale $h$.
\end{proof}

The second instantiation of Lemma~\ref{lem:codewords} guarantees the existence of a \emph{packing code} that will be a key ingredient in defining our hard data distribution instances (Section~\ref{subsec:hard-distributions}).

\begin{lemma}[Packing code]
\label{lem:packing}
There exists a universal constant $c_{\mathrm{pack}}>0$ such that for every sufficiently large integer $d$ there is a set
$
\Theta_d\subseteq\{0,1\}^d
$
with the following properties:
(1) every two distinct elements of $\Theta_d$ have Hamming distance at least $\lfloor d/8\rfloor$; 
and (2) the size is $ \log |\Theta_d|\ge c_{\mathrm{pack}}\,d.$
\end{lemma}

\begin{proof}
Fix $\rho_\star:=3/4$ and let
$
N_d:=\left\lfloor \frac12 \exp\!\left(\frac{9d}{128}\right)\right\rfloor.
$
For all $d$ large enough, $N_d\ge 2$ and
$d\ge 8\rho_\star^{-2}\log(2N_d).$
Therefore, by Lemma~\ref{lem:codewords}, there exists a set of vectors
\[
z_0,\dots,z_{N_d-1}\in\{\pm1\}^d \text{ such that } \bigl|\langle z_a,z_b\rangle\bigr|\le \rho_\star d
\qquad (a\neq b).
\]
Let
\[
\Theta_d:=\{\code{z_a}:0\le a\le N_d-1\}\subseteq\{0,1\}^d.
\]
Then for every distinct $a,b$ (using in the implication that Hamming distance is integer-valued):
\[
\distH(\code{z_a},\code{z_b})
=
\frac{d-\langle z_a,z_b\rangle}{2}
\ge
\frac{1-\rho_\star}{2}d
=
\frac{d}{8}
\implies
\distH(\code{z_a},\code{z_b})\ge \left\lfloor \frac{d}{8}\right\rfloor.
\]
Also, once $d$ is large enough that $\frac12 e^{9d/128}\ge 2$, we have
$
N_d=\left\lfloor \frac12 e^{9d/128}\right\rfloor\ge \frac14 e^{9d/128}
$
and hence
\[
\log |\Theta_d|=\log N_d\ge \frac{9d}{128}-\log 4  \ge \frac{9}{256} d. \qedhere
\]
\end{proof}

\subsection{The hard instance: compressed groups and staircase distributions}

In this subsection, we will describe the two components of our lower-bound instances, whose complexity will be parameterized by an integer  $m$ (which we will take as a power of two). Namely, Section~\ref{subsec:hard-groups} presents our hard group family $G_m$ in Definition~\ref{def:hard-groups}, while Section~\ref{subsec:hard-distributions} constructs the (property-$\Gamma$-dependent) hard data distribution family $\mathcal H_m^\Gamma$ in Definition~\ref{def:hard-family}.

The coding primitives that we just discussed will underlie both of these constructions and their properties: Lemma~\ref{lem:dyadic-signatures} supplies the
short dyadic block signatures from which we build our compressed group family
$G_m$, while Lemma~\ref{lem:packing} supplies the packing code $\ThetaCode$
that indexes the family of hard distributions in $\mathcal H_m^\Gamma$.

\subsubsection{Hard group family \texorpdfstring{$G_m$}{Gm}}
\label{subsec:hard-groups}

We will now introduce our $m$-indexed hard group family $G_m$, and next establish a key combinatorial property of $G_m$: that despite its polylogarithmic size in $m$, it succinctly captures a large (of size proportional to $m$) collection of threshold functions. Our construction is dyadic in nature, and is inspired by~\cite{gopalan2024omnipredictors}.

\begin{definition}[Hard group family $G_m$]
\label{def:hard-groups}
Let $m\ge 16$ be a power of two and $L:=\log_2 m$. Let
\[
\rho_m:=\frac{1}{8(1+L)},
\quad \text{and define }
n_h:=\frac{m}{2^h},
\:
k_h:=\left\lceil 8\rho_m^{-2}\log(2n_h)\right\rceil
\, \text{for all scales $h\in\{0,\dots,L-1\}$. }
\]
With respect to these dyadic parameters, use Lemma~\ref{lem:dyadic-signatures} to supply a low-correlation set of vectors:
\[
z_0^{(h)},\dots,z_{n_h-1}^{(h)}\in\{\pm1\}^{k_h}.
\]
For each $h\in\{0,\dots,L-1\}$ and
$q\in\{1,\dots,k_h\}$, define
\[
\sigma_{h,q}:\{0,1,\dots,m-1\}\to\{\pm1\}
\]
by setting $\sigma_{h,q}(u)=z_a^{(h)}(q)$ on the dyadic block
$[a2^h,(a+1)2^h)$. Define signed probes on $i \in [m]$:
\[
w_0(i):=1,
\qquad
w_{h,q}(i):=\sigma_{h,q}(i-1).
\]
Now, to ensure all groups are binary, split each signed probe into two half-groups:
\[
g_{\mathrm{all}}(i):=1,
\qquad
g_{h,q,+}(i):=\frac{1+w_{h,q}(i)}{2},
\qquad
g_{h,q,-}(i):=\frac{1-w_{h,q}(i)}{2}.
\]
Then, the compressed dyadic sign group family at scale $m$ is defined as:
\[
G_m:=\{g_{\mathrm{all}}\}\cup\{g_{h,q,+},g_{h,q,-}:0\le h\le L-1,\ 1\le q\le k_h\}.
\]
\end{definition}

\begin{observation}[Group family size is polylogarithmic]
\label{obs:group-size}
Because $k_h=O((1+L)^3)$ for all $h$, there is a universal constant
$C_G>0$ such that
\[
|G_m|
=
1+2\sum_{h=0}^{L-1}k_h
\le
C_G(1+\log_2 m)^4.
\]
\end{observation}

We now establish a crucial property of the group family $G_m$: its powerful representational ability. Namely, we will now see that every threshold sign on the ordered domain $[m]$ can be approximated by a short linear combination of the
compressed probes from Definition~\ref{def:hard-groups}. The decomposition of a
prefix into disjoint dyadic intervals used below is a standard discrepancy theory tool; compare the standard reference by
Matou\v{s}ek \cite[Example~2.2]{Matousek1994Range}.

\begin{definition}[Threshold sign function]
For $r\in\{0,1,\dots,m\}$ define the threshold-sign function
\[
\tau_r(u):=
\begin{cases}
+1,&u\le r-1,\\
-1,&u\ge r,
\end{cases}
\qquad u\in\{0,1,\dots,m-1\}.
\]
\end{definition}

\begin{lemma}[Threshold function approximation]
\label{lem:threshold}
Let $m$ be a power of two, let $L:=\log_2 m$, and use the sign probes
$\sigma_{h,q}$ from Definition~\ref{def:hard-groups}. There exist families of
coefficients
\[
(\alpha_0(r))_{r=0}^m,
\qquad
(\alpha_{h,q}(r))_{r=0}^m
\quad (0\le h\le L-1,\ 1\le q\le k_h),
\]
such that, for every $r\in\{0,\dots,m\}$, the function
\[
\widehat\tau_r(u):=\alpha_0(r)+\sum_{h=0}^{L-1}\sum_{q=1}^{k_h}\alpha_{h,q}(r)\sigma_{h,q}(u)
\quad
\text{satisfies}
\quad
\|\tau_r-\widehat\tau_r\|_\infty\le \frac14.
\]
Moreover, the coefficients satisfy the following boundedness conditions:
\[
\sup_{r\in\{0,\dots,m\}} |\alpha_0(r)|\le 1,
\qquad
\sum_{h=0}^{L-1}\sum_{q=1}^{k_h}\sup_{r\in\{0,\dots,m\}} |\alpha_{h,q}(r)|\le 2L.
\]
\end{lemma}

\begin{proof}
For each scale $h$ and block index $0 \leq a \leq n_h-1$, let
$I_{h,a}:=\{a2^h,\dots,(a+1)2^h-1\}$ and let
\[
\widetilde{\ind{\cdot\in I_{h,a}}}(u):=\frac{1}{k_h}\sum_{q=1}^{k_h} z_a^{(h)}(q)\sigma_{h,q}(u)
=\sum_{q=1}^{k_h}\beta_{h,a,q}\sigma_{h,q}(u),
\qquad
\beta_{h,a,q}:=\frac{z_a^{(h)}(q)}{k_h}.
\]
If $u\in I_{h,a}$ then $\widetilde{\ind{\cdot\in I_{h,a}}}(u)=1$, while if
$u\in I_{h,b}$ with $b\neq a$ then
$\bigl|\widetilde{\ind{\cdot\in I_{h,a}}}(u)\bigr|\le \rho_m$. Hence
\[
\bigl\|\ind{\cdot\in I_{h,a}}-\widetilde{\ind{\cdot\in I_{h,a}}}\bigr\|_\infty\le \rho_m,
\qquad
\sum_{q=1}^{k_h}|\beta_{h,a,q}|=1.
\]

For $r=0$ and $r=m$, set $\widehat\tau_0\equiv -1$ and $\widehat\tau_m\equiv 1$;
equivalently, take $\alpha_0(0)=-1$, $\alpha_0(m)=1$, and
$\alpha_{h,q}(0)=\alpha_{h,q}(m)=0$. These two cases already satisfy the claim.

Now fix $1\le r\le m-1$, and write
$r=\sum_{j=1}^t 2^{h_j}$ with $h_1>\cdots>h_t\ge 0$. Set
$s_1:=0$, $s_j:=\sum_{u=1}^{j-1}2^{h_u}$ for $j\ge 2$, and
$I_j:=\{s_j,\dots,s_j+2^{h_j}-1\}$. Then the intervals $I_1,\dots,I_t$ are
disjoint dyadic intervals whose union is $\{0,\dots,r-1\}$, and there is at
most one such interval at each scale. Define
\[
\widehat\chi_r:=\sum_{j=1}^t \widetilde{\ind{\cdot\in I_j}},
\qquad
\chi_r(u):=\ind{u<r}.
\]
Since $\chi_r=\sum_{j=1}^t \ind{\cdot\in I_j}$, we get
\[
\|\chi_r-\widehat\chi_r\|_\infty
\le \sum_{j=1}^t \bigl\|\ind{\cdot\in I_j}-\widetilde{\ind{\cdot\in I_j}}\bigr\|_\infty
\le t\rho_m\le L\rho_m.
\]
Collecting coefficients by scale gives
\[
\widehat\chi_r(u)=\sum_{h=0}^{L-1}\sum_{q=1}^{k_h}\beta_{h,q}(r)\sigma_{h,q}(u),
\]
where, because at most one interval of each scale appears,
$\sup_r |\beta_{h,q}(r)|\le 1/k_h$ and hence
\[
\sum_{q=1}^{k_h}\sup_r |\beta_{h,q}(r)|\le 1
\qquad (0\le h\le L-1).
\]
Summing over $h$ yields
\[
\sum_{h=0}^{L-1}\sum_{q=1}^{k_h}\sup_r |\beta_{h,q}(r)|\le L.
\]

Finally, for $1\le r\le m-1$, set
$\alpha_0(r):=-1$, $\alpha_{h,q}(r):=2\beta_{h,q}(r)$, and
$\widehat\tau_r:=2\widehat\chi_r-1$. Since $\tau_r=2\chi_r-1$,
\[
\|\tau_r-\widehat\tau_r\|_\infty
=2\|\chi_r-\widehat\chi_r\|_\infty
\le 2L\rho_m
\le \frac14.
\]
The coefficient bounds become
$\sup_r |\alpha_0(r)|\le 1$ and
$\sum_{h,q}\sup_r |\alpha_{h,q}(r)|\le 2L$, as required.
\end{proof}

\subsubsection{Hard distribution family}
\label{subsec:hard-distributions}

We now present our hard distribution family construction. Unlike our group family construction $G_m$, which is property-agnostic, the hard distribution families $\mathcal H_m^\Gamma$ will depend on the property $\Gamma$. 

Henceforth, as a useful piece of notation, write $U_m:=\Unif([m])$. At a high level, we build $\mathcal H_m^\Gamma$ from a collection of $\Theta_m$-indexed distributions, where $\Theta_m$ is an $m/2$-dimensional packing code obtained from Lemma~\ref{lem:packing}. Each constituent distribution corresponds to some $\theta \in \Theta_m$, and has $X$ sampled uniformly at random from $\mathcal{X} = [m]$, while the conditional label distributions $(Y | X = i)$, $i \in [m]$, are indexed by a ``parameter map'' --- a \emph{monotonically increasing} sequence of numbers whose jumps are encoded by $\theta$.

\begin{definition}[Hard distribution family $\mathcal H_m^\Gamma$]
\label{def:hard-family}
Consider any regular property $\Gamma$. Fix a regularity witness $\{D_t:t\in I_0\}$ for it, and a nondegenerate interval
$J=[a,b]\subseteq I_0$. 

Fix power-of-two $m\ge 16$. Let $d:=m/2$,
$\gamma_m:=\frac{b-a}{8m}.$
Let $\ThetaCode\subseteq\{0,1\}^d$ be any packing code whose existence has been established by
Lemma~\ref{lem:packing}. 

For each $\theta=(\theta_1,\dots,\theta_d)\in\ThetaCode$, define the \emph{staircase map} $t_\theta$ as:
\[
t_\theta(2j-1):=c_j,
\qquad
t_\theta(2j):=c_j+\gamma_m\theta_j,
\qquad j=1,\dots,d,
\]
where we define an increasing arithmetic progression
$c_j:=a+4\gamma_m(j-1) \: \text{for $j=1,\dots,d$.}$

Each staircase map $t_\theta$ gives rise to an associated distribution $P_\theta$, defined as:
\[
P_\theta^\Gamma:\quad
X\sim \Unif([m]),
\qquad
Y\mid X=i\sim D_{t_\theta(i)}.
\]
Then, the hard distribution family is defined as the collection of all $\Theta$-indexed distributions:
\[
\mathcal H_m^\Gamma:=\{P_\theta^\Gamma:\theta\in\ThetaCode\}.
\]
\end{definition}

To shed light on this construction, Figure~\ref{fig:staircase} illustrates the nature of the staircase map in the special case when $\Gamma$ is the distribution mean. In this case, in fact, the staircase map $t_\theta$ simply corresponds to the regression function $\mu_\theta$.

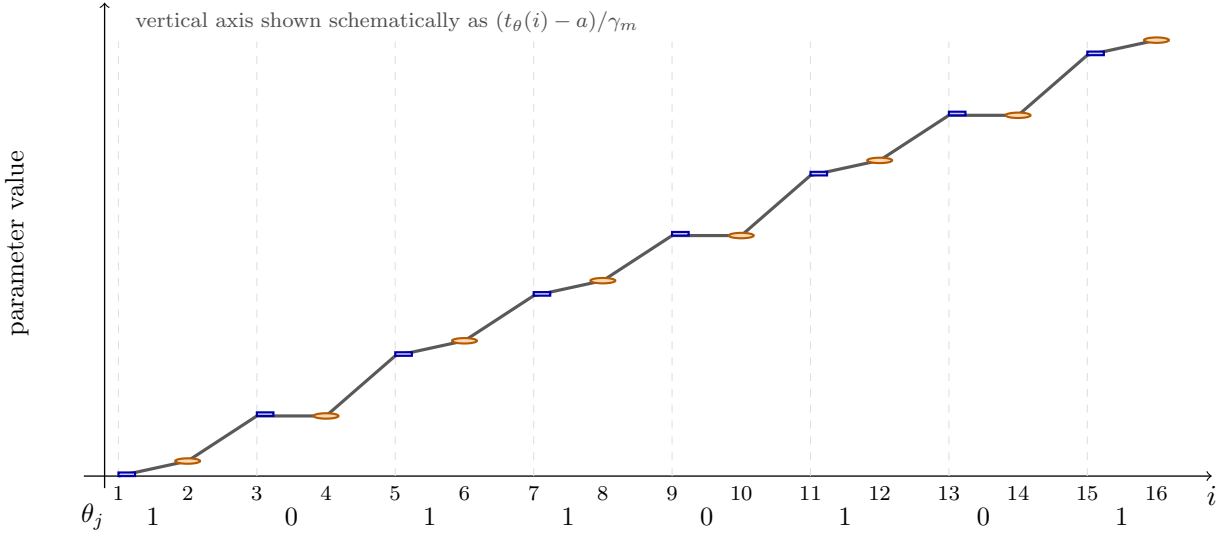
\begin{figure}[t]
\centering
\begin{tikzpicture}[x=0.92cm,y=0.2cm]
  \draw[->, line width=0.5pt] (0.5,0) -- (16.8,0) node[below] {$i$};
  \draw[->, line width=0.5pt] (0.8,-0.8) -- (0.8,31.5);
  \node[rotate=90, anchor=south, font=\small] at (-0.15,15.5) {parameter value};

  \foreach \x/\lab in {1/1,3/3,5/5,7/7,9/9,11/11,13/13,15/15} {
    \draw[gray!25, dashed] (\x,0) -- (\x,{28.9});
    \node[below] at (\x,0) {\scriptsize \lab};
  }
  \foreach \x/\lab in {2/2,4/4,6/6,8/8,10/10,12/12,14/14,16/16} {
    \node[below] at (\x,0) {\scriptsize \lab};
  }

  \draw[very thick, black!65]
    (1,0) -- (2,1) -- (3,4) -- (4,4) -- (5,8) -- (6,9) -- (7,12) -- (8,13)
    -- (9,16) -- (10,16) -- (11,20) -- (12,21) -- (13,24) -- (14,24)
    -- (15,28) -- (16,29);

  \foreach \x/\y in {1/0,3/4,5/8,7/12,9/16,11/20,13/24,15/28} {
    \filldraw[fill=blue!20, draw=blue!65!black, line width=0.8pt] (\x,\y) rectangle ++(0.24,0.24);
  }
  \foreach \x/\y in {2/1,4/4,6/9,8/13,10/16,12/21,14/24,16/29} {
    \filldraw[fill=orange!30, draw=orange!70!black, line width=0.8pt] (\x,\y) circle (0.18);
  }

  \node[anchor=west, font=\scriptsize, text=black!70] at (1.1,30.2) {vertical axis shown schematically as $(t_\theta(i)-a)/\gamma_m$};

  \node[anchor=east, font=\small] at (0.95,-2.7) {$\theta_j$};
  \foreach \x/\bit in {1.5/1,3.5/0,5.5/1,7.5/1,9.5/0,11.5/1,13.5/0,15.5/1} {
    \node[font=\small] at (\x,-2.7) {$\bit$};
  }
\end{tikzpicture}
\caption{A schematic staircase map. When specialized to the mean property, this is
exactly the regression function $\mu_\theta$. The odd indices (blue)
form a base ladder with spacing $4\gamma_m$, while each even index (orange) either stays flat or rises by $\gamma_m$ according to the corresponding
bit $\theta_j$.}
\label{fig:staircase}
\end{figure}

% \begin{observation}
% By Lemma~\ref{lem:packing}, every compressed staircase
% hard family satisfies
% \[
% \log |\ThetaCode|\ge c_1 m
% \qquad\text{with}\qquad
% c_1:=\frac{c_{\mathrm{pack}}}{2}.
% \]
% \end{observation}

We now prove a crucial property of the just defined hard distribution family: that it is well-separated. This fact is powered by the properties of the code $\Theta_m$ underlying the hard family.

\begin{lemma}[Separation in the staircase family]
\label{lem:staircase-separation}
For every distinct $\theta,\theta'\in\ThetaCode$,
\[
\lVert t_\theta-t_{\theta'}\rVert_{L_1(U_m)}\ge \frac{\gamma_m}{16}.
\]
\end{lemma}

\begin{proof}
First note that every $t_\theta$ is nondecreasing and takes values in the selected interval $J$. Indeed, for
each $j$, we have
$t_\theta(2j-1)=c_j\le c_j+\gamma_m\theta_j=t_\theta(2j),$
and for each $j=1,\dots,d-1$, we also have
$t_\theta(2j)\le c_j+\gamma_m<c_j+4\gamma_m=c_{j+1}=t_\theta(2j+1).$
Also, all values lie in $J$ because
\[
t_\theta(m)\le c_d+\gamma_m=a+4\gamma_m(d-1)+\gamma_m=a+2\gamma_m m-3\gamma_m\le a+\frac{b-a}{4}<b.
\]
If $\theta$ and $\theta'$ differ in coordinate $j$, then $t_\theta$ and
$t_{\theta'}$ differ at the location $2j$, by exactly
$\gamma_m$. Thus
\[
\sum_{i=1}^m |t_\theta(i)-t_{\theta'}(i)|=\gamma_m\,\distH(\theta,\theta'),
\qquad
\lVert t_\theta-t_{\theta'}\rVert_{L_1(U_m)}=
\frac{\gamma_m}{m}\,\distH(\theta,\theta').
\]
Because $d=m/2$ is divisible by $8$ and $\ThetaCode$ comes from
Lemma~\ref{lem:packing}, every distinct $\theta,\theta'\in\ThetaCode$ satisfy
\[
\distH(\theta,\theta')\ge \frac{d}{8}=\frac{m}{16}.
\]
Substituting into the displayed identity gives the claimed separation.
\end{proof}

Next, we record that as a consequence of our regularity assumptions, same-size i.i.d.\ samples from any pair of distributions in the hard family must be pairwise close in KL divergence.

\begin{lemma}[Pairwise KL bound for the hard distribution family]
\label{lem:pairwise-kl}
The following $n$-i.i.d.-sample bound holds
for every $\theta,\theta'\in\ThetaCode$:
\[
\KL\bigl((P_\theta^\Gamma)^n\,\|\,(P_{\theta'}^\Gamma)^n\bigr)
\le \frac{C_{\Gamma,\mathrm{KL}}}{2}\,n\gamma_m^2.
\]
\end{lemma}

\begin{proof}
First consider the case of a single ($n=1$) sample. Conditioning on $X$ and using regularity,
\[
\KL(P_\theta^\Gamma\,\|\,P_{\theta'}^\Gamma)
=
\frac1m\sum_{i=1}^m \KL\bigl(D_{t_\theta(i)}\,\|\,D_{t_{\theta'}(i)}\bigr)
\le
\frac{C_{\Gamma,\mathrm{KL}}}{m}\sum_{i=1}^m |t_\theta(i)-t_{\theta'}(i)|^2.
\]
The two parameter maps can differ only at even locations, and each such
location contributes $\gamma_m^2$. Therefore, since $\distH(\theta,\theta')\le d=m/2$, we have
\[
\KL(P_\theta^\Gamma\,\|\,P_{\theta'}^\Gamma)
\le
\frac{C_{\Gamma,\mathrm{KL}}}{m}\,\distH(\theta,\theta')\,\gamma_m^2
\le
\frac{C_{\Gamma,\mathrm{KL}}}{2}\,\gamma_m^2,
\]
The $n$-sample bound then follows by the
additivity of KL for product measures.
\end{proof}

\subsection{Multicalibration implies exact decoding}

Having defined the hard instance, consisting of the hard group family $G_m$ and the hard distribution family, we will now jointly leverage the structural properties of both. In a nutshell, we will now demonstrate that approximate $G_m$-multicalibration on any distribution $P_\theta$ coming from the hard distribution family $\mathcal H_m^\Gamma$ in fact results in \emph{exact recovery} of the hidden parameter $\theta \in \Theta_m$. 

As a stepping stone to this, we will first prove that a learner's \emph{prediction error} on the underlying distribution will be at most a $\log m$ factor away from the learner's $G_m$-multicalibration error. In other words, $G_m$-multicalibration implies low prediction error on instances that we have defined.

% In this subsection, we show that sufficiently small multicalibration error on the
% hard group family $G_m$ leaves so little room for coarsening that the hidden
% codeword $\theta$ --- which indexes the ground-truth distribution inside the hard
% family --- becomes recoverable.

To formalize this, fix a nondecreasing map $t:[m]\to I_0$. The reader may think of this map as
$t=t_\theta$, one of the staircase maps from
Definition~\ref{def:hard-family}, but the following proposition only, and crucially, uses the map's
monotonicity. 

Let $P_t$ denote the induced distribution $X\sim U_m, \:
(Y\mid X=i) \sim D_{t(i)},$
and write
$\MC_t^\Gamma(Q;G_m):=\MC_{P_t}^\Gamma(Q;G_m).$
Finally, recall (see Definition~\ref{def:prediction-error}) that for a predictor $Q=(Q_1,\dots,Q_m)$ on $[m]$, its
prediction error is
\[
\Delta_t(Q):=\frac1m\sum_{i=1}^m\int |v-t(i)|\,Q_i(dv).
\]

\begin{proposition}[Multicalibration implies low prediction error]
\label{prop:anticoarsen}
Let $t:[m]\to I_0$ be nondecreasing. Then every randomized predictor
$Q=(Q_1,\dots,Q_m)$ on $[m]$ with finite prediction support satisfies
\[
\Delta_t(Q)\le \frac{6(1+\log_2 m)}{c_\Gamma}\,\MC_t^\Gamma(Q;G_m).
\]
\end{proposition}

\begin{proof}
Write $L:=\log_2 m$. For each prediction value $v\in V(Q)$ and bounded signed
weight $w:[m]\to\Rbb$, define
\[
B_v(w):=\frac1m\sum_{i=1}^m w(i)Q_i(\{v\})M_\Gamma(v,t(i)),
\qquad
\Err(w):=\sum_{v'\in V(Q)}\bigl|B_{v'}(w)\bigr|.
\]
For binary groups $g:[m]\to\{0,1\}$, this notation agrees with the
population quantity under $P_t$:
\[
\Err(g)=\Err_{P_t}^\Gamma(Q;g).
\]

For each $v\in V(Q)$, let
$r(v):=|\{i\in[m]:t(i)\le v\}|$, define
$s_v(i):=\tau_{r(v)}(i-1)$ and $\widehat s_v(i):=\widehat\tau_{r(v)}(i-1)$, and note
that $s_v(i)$ is the sign of $v-t(i)$ because $t$ is nondecreasing. When
$v=t(i)$, regularity gives $M_\Gamma(v,t(i))=0$, so in all cases
\[
s_v(i)M_\Gamma(v,t(i))=|M_\Gamma(v,t(i))|\ge c_\Gamma |v-t(i)|.
\]
Hence, with
\[
A_v:=B_v(s_v)=\frac1m\sum_{i=1}^m Q_i(\{v\})|M_\Gamma(v,t(i))|,
\]
we have
\[
\Delta_t(Q)\le \frac{1}{c_\Gamma}\sum_{v\in V(Q)} A_v.
\]

Now let $e_v:=s_v-\widehat s_v$. By Lemma~\ref{lem:threshold},
$\|e_v\|_\infty\le 1/4$, so
\[
\bigl|B_v(e_v)\bigr|
\le \frac1m\sum_{i=1}^m |e_v(i)|Q_i(\{v\})|M_\Gamma(v,t(i))|
\le \frac14 A_v.
\]
Since $B_v(s_v)=B_v(\widehat s_v)+B_v(e_v)$,
this implies
$\frac34 A_v\le \bigl|B_v(\widehat s_v)\bigr|$.
Summing over $v$ gives
\[
\Delta_t(Q)
\le \frac{4}{3c_\Gamma}\sum_{v\in V(Q)} \bigl|B_v(\widehat s_v)\bigr|.
\]

Writing
\[
\widehat s_v(i)=\alpha_0(r(v))w_0(i)+\sum_{h=0}^{L-1}\sum_{q=1}^{k_h}\alpha_{h,q}(r(v))w_{h,q}(i),
\]
linearity of $B_v(\cdot)$ and the coefficient bounds from
Lemma~\ref{lem:threshold} give
\[
\sum_{v\in V(Q)} \bigl|B_v(\widehat s_v)\bigr|
\le \sup_r |\alpha_0(r)|\,\Err(w_0)
+\sum_{h=0}^{L-1}\sum_{q=1}^{k_h}\sup_r |\alpha_{h,q}(r)|\,\Err(w_{h,q}).
\]

Now $w_0=g_{\mathrm{all}}$, so
\[
\Err(w_0)=\Err_{P_t}^\Gamma(Q;g_{\mathrm{all}})\le \MC_t^\Gamma(Q;G_m).
\]
Also, for each $0\le h\le L-1$, each $1\le q\le k_h$, and each $v\in V(Q)$, we have
$w_{h,q}=g_{h,q,+}-g_{h,q,-}$ and hence
\[
B_v(w_{h,q})
=
B_v(g_{h,q,+})-B_v(g_{h,q,-}).
\]
Therefore
\[
\bigl|B_v(w_{h,q})\bigr|
\le
\bigl|B_v(g_{h,q,+})\bigr|
+
\bigl|B_v(g_{h,q,-})\bigr|,
\]
and summing over $v$ gives
\[
\Err(w_{h,q})
\le
\Err_{P_t}^\Gamma(Q;g_{h,q,+})
+
\Err_{P_t}^\Gamma(Q;g_{h,q,-})
\le 2\MC_t^\Gamma(Q;G_m).
\]

Thus, using again the coefficient bounds from Lemma~\ref{lem:threshold}, the
right-hand side is at most
$(1+4L)\MC_t^\Gamma(Q;G_m)$. Therefore, as claimed, we obtain
\[
\Delta_t(Q)
\le \frac{4}{3c_\Gamma}(1+4L)\MC_t^\Gamma(Q;G_m)
\le \frac{6(1+L)}{c_\Gamma}\MC_t^\Gamma(Q;G_m). \qedhere
\]
\end{proof}

Now, we are ready to show that on our hard instance, multicalibration in fact implies exact decoding. Indeed, as we just saw, a multicalibrated learner must have low prediction error on our instance. Now, a simple geometric argument leveraging the $L_1(U_m)$ well-separation property of the hard distribution family --- which we proved earlier in
Lemma~\ref{lem:staircase-separation} --- will be seen to turn low prediction error into exact recovery of the underlying distribution parameter $\theta \in \Theta_m$.

\begin{proposition}[Multicalibration implies exact decoding]
\label{prop:decode}
Given a finitely supported randomized predictor $Q=(Q_1,\dots,Q_m)$, define its mean prediction function by
\[
\bar f_Q(i):=\int v\,Q_i(dv),
\]
and consider any nearest-neighbor decoder in $L_1(U_m)$:
\[
\widehat\theta_\Gamma(Q)\in\arg\min_{\vartheta\in\ThetaCode}\lVert t_{\vartheta}-\bar f_Q\rVert_{L_1(U_m)}.
\]
Then,
\[
\MC_{t_\theta}^\Gamma(Q;G_m)\le \frac{c_\Gamma\gamma_m}{384(1+\log_2 m)} \quad \text{implies} \quad  \widehat\theta_\Gamma(Q)=\theta.
\]
\end{proposition}

\begin{proof}
By Jensen's inequality,
\[
|\bar f_Q(i)-t_\theta(i)|\le \int |v-t_\theta(i)|\,Q_i(dv).
\]
Summing over $i$ gives
\[
\lVert\bar f_Q-t_\theta\rVert_{L_1(U_m)}\le \Delta_{t_\theta}(Q).
\]
By construction $t_\theta$ is nondecreasing, so Proposition~\ref{prop:anticoarsen} applies and gives
\[
\lVert\bar f_Q-t_\theta\rVert_{L_1(U_m)}
\le \frac{6(1+\log_2 m)}{c_\Gamma}\,\MC_{t_\theta}^\Gamma(Q;G_m)
\le \frac{\gamma_m}{64}.
\]
However, by Lemma~\ref{lem:staircase-separation}, for every $\theta'\neq \theta$, we have
$\lVert t_{\theta'}-t_\theta\rVert_{L_1(U_m)}\ge \frac{\gamma_m}{16},$
and therefore
\[
\lVert\bar f_Q-t_{\theta'}\rVert_{L_1(U_m)}
\ge \lVert t_{\theta'}-t_\theta\rVert_{L_1(U_m)}-\lVert\bar f_Q-t_\theta\rVert_{L_1(U_m)}
\ge \frac{\gamma_m}{16}-\frac{\gamma_m}{64}
>\frac{\gamma_m}{64}.
\]
So $t_\theta$ is the unique nearest codeword to $\bar f_Q$, and hence $\widehat\theta_\Gamma(Q)=\theta$.
\end{proof}

\subsection{Exact decoding has high sample complexity}

We have just seen that sufficiently multicalibrated learners must be exact decoders of the hidden distribution parameter $\theta \in \Theta_m$. We will now show that in order to exactly decode $\theta$, the learner's sample complexity must be large.

Specifically, recall the pairwise-(KL-)closeness property of the hard distribution family (Lemma~\ref{lem:pairwise-kl}). Applying Fano's inequality together with this property, we will demonstrate that exact decoding/learning of the family requires sample complexity on the order of $m/\gamma_m^2$.

\begin{lemma}[Exact decoding has high sample complexity]
\label{lem:fano-close}
Suppose a learner gets $n$ i.i.d.\ samples from an unknown distribution in
$\mathcal H_m^\Gamma$ and outputs predictors $Q_{S,\zeta}$ such that for every
$\theta\in\ThetaCode$,
\[
\Pbb_{S\sim (P_\theta^\Gamma)^n,\,\zeta}\!\left(
\lVert \bar f_{Q_{S,\zeta}}-t_\theta\rVert_{L_1(U_m)}\le \frac{\gamma_m}{64}
\right)\ge \frac23.
\]
Then, $n\ge c_{2,\Gamma}\,\frac{m}{\gamma_m^2},$ where  $c_{2,\Gamma}>0$ is a constant that depends only on $\Gamma$.
\end{lemma}

\begin{proof}
Define the nearest-neighbor decoder
\[
\widetilde\theta(S,\zeta)\in
\arg\min_{\vartheta\in\ThetaCode}
\lVert t_{\vartheta}-\bar f_{Q_{S,\zeta}}\rVert_{L_1(U_m)}.
\]
Lemma~\ref{lem:staircase-separation} implies that on the event
\[
\lVert \bar f_{Q_{S,\zeta}}-t_\theta\rVert_{L_1(U_m)}\le \frac{\gamma_m}{64},
\]
$t_\theta$ is the unique
nearest codeword to $\bar f_{Q_{S,\zeta}}$, so
$\widetilde\theta(S,\zeta)=\theta$. Therefore
\[
\sup_{\theta\in\ThetaCode}
\Pbb_{S\sim (P_\theta^\Gamma)^n,\,\zeta}\!\left(
\widetilde\theta(S,\zeta)\neq \theta
\right)\le \frac13.
\]
Now let $R_\zeta$ be the law of the learner's seed, and view $(S,\zeta)$ as one sample from the product family
\[
\{(P_\theta^\Gamma)^n\otimes R_\zeta:\theta\in\ThetaCode\}.
\]
Let $\Theta$ be uniform
on $\ThetaCode$, let $Z$ denote the resulting observation, and let
$\widehat\Theta:=\widetilde\theta(Z)$. Write
\[
\overline Q:=\frac{1}{|\ThetaCode|}\sum_{\theta\in\ThetaCode}
\bigl((P_\theta^\Gamma)^n\otimes R_\zeta\bigr).
\]
By Fano's inequality
\cite{CoverThomas2006},
\[
\frac{1}{|\ThetaCode|}\sum_{\theta\in\ThetaCode}
\Pbb_{S\sim (P_\theta^\Gamma)^n,\,\zeta}\!\left(\widetilde\theta(S,\zeta)\neq\theta\right)
\ge
1-\frac{I(\Theta;Z)+\log 2}{\log |\ThetaCode|}.
\]
Also, by convexity of KL in its second argument,
\[
I(\Theta;Z)
=
\frac{1}{|\ThetaCode|}\sum_{\theta\in\ThetaCode}
\KL\bigl((P_\theta^\Gamma)^n\otimes R_\zeta\,\|\,\overline Q\bigr)
\le
\frac{1}{|\ThetaCode|^2}
\sum_{\theta,\theta'\in\ThetaCode}
\KL\bigl((P_\theta^\Gamma)^n\otimes R_\zeta\,\|\,(P_{\theta'}^\Gamma)^n\otimes R_\zeta\bigr),
\]
Because $R_\zeta$ does not depend on
$\theta$,
\[
\KL\bigl((P_\theta^\Gamma)^n\otimes R_\zeta\,\|\,(P_{\theta'}^\Gamma)^n\otimes R_\zeta\bigr)
=
\KL\bigl((P_\theta^\Gamma)^n\,\|\,(P_{\theta'}^\Gamma)^n\bigr)
\le \frac{C_{\Gamma,\mathrm{KL}}}{2}\,n\gamma_m^2
\]
for every $\theta,\theta'\in\ThetaCode$, by Lemma~\ref{lem:pairwise-kl}.
Therefore
\[
I(\Theta;Z)\le \frac{C_{\Gamma,\mathrm{KL}}}{2}\,n\gamma_m^2.
\]
Now note that the worst-case error is at least the average
error. Then, using that $\log |\ThetaCode|\ge c_1m$,
\[
\sup_{\theta\in\ThetaCode}
\Pbb_{S\sim (P_\theta^\Gamma)^n,\,\zeta}\!\left(\widetilde\theta(S,\zeta)\neq\theta\right)
\ge 1-\frac{\frac{C_{\Gamma,\mathrm{KL}}}{2}n\gamma_m^2+\log 2}{c_1m}.
\]
Choose $m_0$ large enough that $\log 2\le c_1m/4$ for all $m\ge m_0$. Then for
$m\ge m_0$, if
$n\le \frac{c_1}{2C_{\Gamma,\mathrm{KL}}}\cdot\frac{m}{\gamma_m^2},$
the Fano lower bound becomes at least $1/2$, contradicting the displayed
worst-case error bound $1/3$. Hence the claim holds with $c_{2,\Gamma}:=\frac{c_1}{2C_{\Gamma,\mathrm{KL}}}.$
\end{proof}

\subsection{Proof of Theorem~\ref{thm:lower-regular}}

We now finish the lower bound proof. First, we record a consequence of H\"older's inequality: that a lower bound for ECE multicalibration will imply a lower bound for $L_p$-multicalibration, for every $p \in 1[, 2]$. In fact, as we will later see, the implied lower bound will also be optimal for all $p \in [1, 2]$.

\begin{lemma}[From $L_p$ to ECE multicalibration]
\label{lem:lower-holder}
For every distribution $P$ on a context--label space, every finite family
$G$ of binary groups, and every finitely supported randomized predictor $Q$,
\[
\MC_P^\Gamma(Q;G)\le \Bigl(\MC_P^{\Gamma,(p)}(Q;G)\Bigr)^{1/p}.
\]
\end{lemma}

\begin{proof}
If $p=1$, this is immediate from the definition. Assume $p>1$, fix a group
$g\in G$, and write
\[
a_v:=|B_P^\Gamma(Q;v,g)|,
\qquad
\pi_v:=\pi_P(Q;v)
\]
on the support where $\pi_v>0$. By H\"older's inequality,
\[
\sum_v a_v
=
\sum_v \left(\frac{a_v}{\pi_v^{(p-1)/p}}\right)\pi_v^{(p-1)/p}
\le
\left(\sum_v \frac{a_v^p}{\pi_v^{p-1}}\right)^{1/p}\left(\sum_v \pi_v\right)^{(p-1)/p}
=
\Bigl(\Err_P^{\Gamma,(p)}(Q;g)\Bigr)^{1/p}.
\]
Taking the maximum over $g\in G$ gives the claim.
\end{proof}

\begin{proof}[Proof of Theorem~\ref{thm:lower-regular}]
Fix a regular property $\Gamma$ and $p\in[1,2]$. Fix a
regularity witness $\{D_t:t\in I_0\}$ and interval $J=[a,b]$.
Let $c_\Gamma$, $C_{\Gamma,\mathrm{KL}}$ be the corresponding regularity
constants. Set
$
c_{3,\Gamma}:=\frac{c_\Gamma}{384}.
$

Choose the largest power of two $m$ such that
\[
\left(c_{3,\Gamma}\,\frac{\gamma_m}{1+\log_2 m}\right)^p\ge \varepsilon,
\qquad
\gamma_m:=\frac{b-a}{8m}.
\]
Since $\gamma_m=\Theta_\Gamma(1/m)$, this choice gives
\[
m=\Theta_\Gamma\!\left(\frac{(1/\varepsilon)^{1/p}}{\log(1/\varepsilon)}\right).
\]
Let $\mathcal H_m^\Gamma=\{P_\theta^\Gamma:\theta\in\ThetaCode\}$ be the
hard distribution family from Definition~\ref{def:hard-family}, and let
$G_m$ be the hard group family from
Definition~\ref{def:hard-groups}. By the bound after
Definition~\ref{def:hard-groups}, the group family has size
\[
|G_m|=O\!\left((1+\log_2 m)^4\right)=O\!\left(\log^4(1/\varepsilon)\right).
\]

Assume for contradiction that a learner with internal random seed $\zeta$, given $G_m$ and $n$ i.i.d. samples
from an unknown distribution in $\mathcal H_m^\Gamma$, outputs a predictor
$Q_{S,\zeta}$ satisfying
\[
\Pbb_{S\sim (P_\theta^\Gamma)^n,\,\zeta}\!\left(\MC_{P_\theta^\Gamma}^{\Gamma,(p)}(Q_{S,\zeta};G_m)\le \varepsilon\right)\ge \frac23
\qquad\text{for every }\theta\in\ThetaCode.
\]
By Lemma~\ref{lem:lower-holder}, on the displayed success event we also have
\[
\MC_{P_\theta^\Gamma}^{\Gamma}(Q_{S,\zeta};G_m)
\le \varepsilon^{1/p}
\le c_{3,\Gamma}\,\frac{\gamma_m}{1+\log_2 m}
=
\frac{c_\Gamma\gamma_m}{384(1+\log_2 m)}.
\]
For the hard family this is exactly
\[
\MC_{t_\theta}^{\Gamma}(Q_{S,\zeta};G_m)
\le
\frac{c_\Gamma\gamma_m}{384(1+\log_2 m)}.
\]
Since $t_\theta$ is nondecreasing, Proposition~\ref{prop:anticoarsen} and
Jensen's inequality imply that on the same event,
\[
\lVert \bar f_{Q_{S,\zeta}}-t_\theta\rVert_{L_1(U_m)}
\le \Delta_{t_\theta}(Q_{S,\zeta})
\le \frac{6(1+\log_2 m)}{c_\Gamma}\,\MC_{t_\theta}^{\Gamma}(Q_{S,\zeta};G_m)
\le \frac{\gamma_m}{64}.
\]
Thus the hypothesis of Lemma~\ref{lem:fano-close} is satisfied, and we obtain
\[
n\ge c_{2,\Gamma}\,\frac{m}{\gamma_m^2}=\Omega_\Gamma(m^3).
\]
Substituting the chosen value of $m$ then gives the claimed lower bound:
\[
n=\Omega_\Gamma\!\left(\frac{1}{\varepsilon^{3/p}\log^3(1/\varepsilon)}\right)
=\widetilde\Omega_\Gamma(\varepsilon^{-3/p}). \qedhere
\]
\end{proof}

\section{Upper Bounds}
\label{sec:upper-bounds}

This section proves a constructive upper bound that covers ordinary (and in fact, even swap) weighted
$L_p$ property multicalibration, for every finite $p\ge 1$ and every property with a bounded identification function (including limited to mean multicalibration). 
The main result is given in Theorem~\ref{thm:upper-main}. For $1\le p\le 2$, the upper bound is $\widetilde{O}(\epsilon^{-3/p})$: this exponent matches the lower-bound exponent of
Section~\ref{sec:lower-bounds} up to polylogarithmic factors, and thus implies the tightness of our sample complexity bounds for property $L_p$-multicalibration for all $p \in [1, 2]$. In addition, Theorem~\ref{thm:upper-main} gives an upper bound for $p > 2$, which is $\widetilde{O}(\epsilon^{-(p+1)/p})$; in this regime, Section~\ref{sec:lower-bounds} does not supply a matching lower bound, and therefore, the tight sample complexity of $L_p$-multicalibration for $p > 2$ is left as an open question.

Before stating the next theorem, we note that it is stated in terms of the (population) \emph{swap multicalibration} metric (SMC), which is stronger than multicalibration due to an interchange of the summation and the maximum operator. By definition, it is the population analog of the \emph{empirical SMC}, which is formalized below in Definition~\ref{def:online-notation}.

\begin{theorem}[General Multicalibration Upper Bound]
\label{thm:upper-main}
Fix $p\ge 1$. Let $P$ be a distribution on $\mathcal X\times \mathcal{Y}$. Fix any $[0, 1]$-valued property $\Gamma$ with a bounded (w.l.o.g.\ by $1$ in absolute value) identification function $V$. Suppose the expected conditional identification function
$m_x(v):=\Ebb[V(v,Y)\mid X=x]$
is $\rho$-Lipschitz on $[0,1]$ and satisfies
$m_x(0)\le 0\le m_x(1)$ 
for $P_X$-almost every
$x\in\mathcal X$.

Then there exists a randomized batch learner with the following sample-complexity
guarantee. For every $\varepsilon>0$, if the learner is given any
finite family of binary groups $g:\mathcal X\to\{0,1\}$, denoted by $G$, together with $T$ i.i.d.\ samples from $P$, where
\[
T = \widetilde O_{p,\rho}\!\left(\varepsilon^{-3/p}\log^{3/2}(2|G|)\right)
\; (\text{for } p\in [1,2]),
\quad \text{and } \:
T = \widetilde O_{p,\rho}\!\left(\varepsilon^{-(p+1)/p}\log^{(p+1)/2}(2|G|)\right)
\; (\text{for } p\ge 2),
\]
then it outputs a finitely supported randomized predictor $Q$ such that
\[
\Pbb\bigl(\MC_P^{\Gamma,(p)}(Q;G)\le \varepsilon\bigr)\ge \Pbb\bigl(\SMC_P^{\Gamma,(p)}(Q;G)\le \varepsilon\bigr)\ge \frac23.
\]
If $|G|\le \varepsilon^{-O(1)}$, this implies
$T = \widetilde O_{p,\rho}(\varepsilon^{-3/p})
\: (1\le p\le 2),$ and
$T = \widetilde O_{p,\rho}(\varepsilon^{-(p+1)/p})
\: (p\ge 2).$
\end{theorem}

\begin{observation}[Theorem~\ref{thm:upper-main} subsumes mean multicalibration]
For the mean property, $V(v,y)=v-y$ and $m_x(v)=\Ebb[v-Y\mid X=x]=v-\Ebb[Y\mid X=x].$ Hence, it satisfies the conditions of Theorem~\ref{thm:upper-main}, with Lipschitz parameter $\rho=1$.
\end{observation}

The proof of Theorem~\ref{thm:upper-main} is provided below, in two
steps. In Section~\ref{subsec:upper-otb}, we prove a general online-to-batch reduction that converts an empirical swap
multicalibration guarantee on a grid into a population guarantee for the averaged batch
predictor. In Section~\ref{subsec:upper-plug-in}, we then instantiate the batch predictor by conversion of the online algorithm of Hu et al.\ \cite[Corollary~1]{hu2025efficient}.

\paragraph{A sharper upper bound for mean ECE multicalibration}

As it turns out, at $p=1$, the multicalibration upper bound can be obtained more directly, and with a better logarithmic dependence on $|G|$, via an online-to-batch conversion of the online $L_1$ ECE-multicalibration algorithm of Noarov et al.~\cite[Theorem~2.4]{noarov2025high} --- the first $O(T^{2/3})$ online multicalibration algorithm.

The online-to-batch conversion here does not require the technicalities of going through swap multicalibration or adaptive Freedman-type concentration. The main technical reasons for the simplification are that for $p=1$ in particular, unlike for other $L_p$ metrics, (1) duality enables us to \emph{linearize} the ``inner component'' of the mulicalibration metric, and (2) there are no bucket mass-dependent denominators that could blow up the error. As such, the online-to-batch reduction can proceed via a standard Azuma-Hoeffding concentration argument.

The result is presented below as Theorem~\ref{cor:mean-ece-sharper}, and the proof is provided in Section~\ref{subsec:upper-mean-sharp}.
Thus, readers interested in the proof of the upper bound just for the canonical (mean ECE) multicalibration case, or wishing to gain intuition for the proof of the general upper bound, may proceed directly to Section~\ref{subsec:upper-mean-sharp} after reading Section~\ref{subsec:upper-setup}.

\begin{theorem}[Sharper Mean-ECE Multicalibration Upper Bound]
\label{cor:mean-ece-sharper}
There is a randomized batch learner which, given any finite group family $G$, outputs a predictor whose population mean ECE multicalibration error is at most $\varepsilon$ with
probability at least $2/3$,
and which uses only
\[
T=\widetilde O\!\left(\varepsilon^{-3}\log(2|G|)\right)
\]
i.i.d.\ samples. In particular, if $|G|\le \varepsilon^{-O(1)}$, the learner's sample complexity is
$\widetilde O(\varepsilon^{-3})$.
\end{theorem}

\subsection{Online-setting notation and constructing the averaged predictor}
\label{subsec:upper-setup}

\paragraph{Notation} We now introduce the transcript notation used throughout the upper-bound proofs.

\begin{definition}[Online notation]
\label{def:online-notation}
Fix a finite prediction grid
$\Lambda=\{v_1,\dots,v_K\}\subseteq[0,1],$
and let 
\[q_t:\mathcal X\to\Delta(\Lambda)\]
be roundwise randomized predictions. For an
i.i.d.\ transcript
$S=((X_1,Y_1),\dots,(X_T,Y_T))\sim P^T,$
define the empirical bucket masses and empirical bucket biases by
\[
\widehat\pi_{T,k}(S):=\frac1T\sum_{t=1}^T q_{t,k}(X_t),
\qquad
\widehat B_{T,k}^{\Gamma}(S;g):=\frac1T\sum_{t=1}^T g(X_t)q_{t,k}(X_t)V(v_k,Y_t).
\]
The empirical ordinary and swap weighted $L_p$ transcript objectives are
\[
\widehat{\MC}_T^{\Gamma,(p)}(S;G,\Lambda)
:=
\max_{g\in G} \mkern-30mu
\sum_{k\in[K]:\,\widehat\pi_{T,k}(S)>0} \mkern-20mu
\tfrac{\bigl|\widehat B_{T,k}^{\Gamma}(S;g)\bigr|^p}{\widehat\pi_{T,k}(S)^{p-1}},
\quad
\widehat{\SMC}_T^{\Gamma,(p)}(S;G,\Lambda)
:= \mkern-30mu
\sum_{k\in[K]:\,\widehat\pi_{T,k}(S)>0} \mkern-20mu
\max_{g\in G}
\tfrac{\bigl|\widehat B_{T,k}^{\Gamma}(S;g)\bigr|^p}{\widehat\pi_{T,k}(S)^{p-1}}.
\]
\end{definition}

We now recall (and briefly reprove for completeness) that swap multicalibration yields guarantees that are stronger than those of ordinary multicalibration.

\begin{lemma}[Swap multicalibration implies ordinary multicalibration]
\label{lem:swap-dominates}
For every $p\ge 1$,
\[
\MC_P^{\Gamma,(p)}(Q_S;G)
\le
\SMC_P^{\Gamma,(p)}(Q_S;G),
\qquad \text{and} \qquad
\widehat{\MC}_T^{\Gamma,(p)}(S;G,\Lambda)
\le
\widehat{\SMC}_T^{\Gamma,(p)}(S;G,\Lambda).
\]
\end{lemma}

\begin{proof}
We prove the first claim; the second follows identically. For each bucket $k$ and group $g$, if $\pi_{P,k}(Q_S)=0$ then let $a_k(g):= 0$, and if if $\pi_{P,k}(Q_S)>0$ then let $a_k(g):= \dfrac{\bigl|B^{\Gamma}_{P,k}(Q_S;g)\bigr|^p}{\pi_{P,k}(Q_S)^{p-1}}$.
Then,
\[
\MC_P^{\Gamma,(p)}(Q_S;G)= \max_{g\in G}\sum_k a_k(g)
\le
\sum_k \max_{g\in G} a_k(g) = \SMC_P^{\Gamma,(p)}(Q_S;G). \qedhere
\]
\end{proof}

\paragraph{Constructing a batch predictor from an online predictor} Now, we formalize how an online predictor should be converted to a batch predictor. This requires a bit more care than in the classical online-to-batch reduction for no-regret learning algorithms: rather than simply sampling the predictor from a random round in the transcript, we need to average the predictors over all rounds. Formally, this yields the following construction.

\begin{definition}[Batch predictor]
Define the averaged batch predictor $Q_S$ with support $\Lambda$ by
\[
q_{S,k}(x):=\frac1T\sum_{t=1}^T q_{t,k}(x),
\qquad
(Q_S)_x:=\sum_{k=1}^K q_{S,k}(x)\,\delta_{v_k}.
\]
Its population bucket masses and population bucket biases are
\[
\pi_{P,k}(Q_S):=\Ebb_{(X,Y)\sim P}\bigl[q_{S,k}(X)\bigr],
\qquad
B_{P,k}^{\Gamma}(Q_S;g):=\Ebb_{(X,Y)\sim P}\bigl[g(X)q_{S,k}(X)V(v_k,Y)\bigr].
\]
\end{definition}

\subsection{The online-to-batch reduction for swap multicalibration}
\label{subsec:upper-otb}

First, we prove a form of Freedman's inequality that incorporates the (random) second-moment martingale difference bounds into the bound. This lemma packages the Freedman-plus-dyadic-peeling argument of~\cite[Lemma~6 and the proof of Theorem~2, esp.\ equations~(8)--(9)]{hu2025efficient}.

\begin{lemma}[Variance process-adaptive Freedman bound]
\label{lem:adaptive-freedman}
Let $(\mathcal F_t)_{t=0}^T$ be a filtration, let $b\ge 1$, and let $Z_1,\dots,Z_T$ be a
martingale difference sequence adapted to it with $|Z_t|\le b$ for all $t$. 

Let $u_1,\dots,u_T$ be $\mathcal F_{t-1}$-measurable random variables in $[0,1]$ such that
\[
\Ebb[Z_t^2\mid \mathcal F_{t-1}]\le u_t
\qquad (t\in[T]).
\]
Then for every $L\ge 1$, for a universal $C > 0$ it holds that
\[
\Pbb\!\left(
\left|\frac1T\sum_{t=1}^T Z_t\right|
>
C\left(\sqrt{\frac{\pi L}{T}}+\frac{bL}{T}\right)
\right)
\le
2(\lceil\log_2 T\rceil+1)e^{-L}, \quad \text{where } \pi:=\frac1T\sum_{t=1}^T u_t.
\]
\end{lemma}

\begin{proof}
Let
$S_t:=\sum_{s=1}^t Z_s,
\:
W_t:=\sum_{s=1}^t \Ebb[Z_s^2\mid \mathcal F_{s-1}].$
Then
$W_T\le \sum_{t=1}^T u_t = T\pi.$

Set $J:=\lceil\log_2 T\rceil$. For $j=0,\dots,J-1$, let
$I_j:=(2^{-j-1},2^{-j}],$ and let $I_J:=[0,2^{-J}]$. 

\noindent
(The standard form of) Freedman's inequality gives, for every $v>0$ and every $L\ge 1$,
\[
\Pbb\bigl(|S_T|>2\sqrt{vL}+2bL,\ W_T\le v\bigr)\le 2e^{-L}.
\]
If $\pi\in I_j$ with $j\le J-1$, then $W_T\le T2^{-j}$, and hence
$\Pbb\bigl(|S_T|>2\sqrt{T2^{-j}L}+2bL,\ \pi\in I_j\bigr)\le 2e^{-L}.$
If $\pi\in I_J$, then $2^{-J}\le 1/T$, so $W_T\le 1$, and therefore
$\Pbb\bigl(|S_T|>2\sqrt{L}+2bL,\ \pi\in I_J\bigr)\le 2e^{-L}.$

A union bound over $0 \leq j \leq J$ shows that the corresponding bound holds on all intervals at once, with probability at least
$1-2(J+1)e^{-L}.$

On that event, if $\pi\in I_j$ with
$j\le J-1$, then $2^{-j}\le 2\pi$, and hence
$\left|\frac{S_T}{T}\right|
\le
2\sqrt{\frac{2\pi L}{T}}+\frac{2bL}{T}.$

If instead $\pi\in I_J$, then
$\left|\frac{S_T}{T}\right|
\le
\frac{2\sqrt L}{T}+\frac{2bL}{T}
\le
\frac{4bL}{T},$
because $b\ge 1$ and $L\ge 1$ imply $\sqrt L\le L\le bL$. Enlarging the
constant gives the claim.
\end{proof}

With this concentration bound in hand, we are ready to establish the bridge between the empirical online, and the population, swap multicalibration.

\begin{proposition}[Online-to-batch reduction for swap $L_p$-multicalibration]
\label{prop:otb-swap-all-p}
Fix $p\ge 1$, and let $Q_S$ be the averaged batch predictor defined above. There is a
constant $C_p>0$ such that, with $L:=\log(4K|G|T),$
we have
\[
\Ebb_{S\sim P^T}\!\left[\SMC_P^{\Gamma,(p)}(Q_S;G)\right]
\le
C_p\,\Ebb_{S\sim P^T}\!\left[\widehat{\SMC}_T^{\Gamma,(p)}(S;G,\Lambda)\right]
+
\begin{cases}
C_p\left[\left(\dfrac{KL}{T}\right)^{p/2}+\dfrac{KL}{T}\right], & 1\le p\le 2,\\[1.2ex]
C_p\dfrac{KL}{T}, & p\ge 2.
\end{cases}
\]
Consequently, the same bound holds for $\Ebb_{S\sim P^T}\!\left[\MC_P^{\Gamma,(p)}(Q_S;G)\right]$.
\end{proposition}

\begin{proof}
Define for $k\in[K]$ the population bucket mass, and for $g\in G$, the population bucket bias:
\[
\pi_k(S):=\Ebb_{(X,Y)\sim P}[q_{S,k}(X)] \qquad \text{ and } \qquad B_k(S;g):=\Ebb_{(X,Y)\sim P}[g(X)q_{S,k}(X)V(v_k,Y)].
\]
For the empirical quantities, write the shorthands:
\[
\widehat\pi_k(S):=\widehat\pi_{T,k}(S),
\qquad
\widehat B_k(S;g):=\widehat B^{\Gamma}_{T,k}(S;g),
\]
Thus, in this notation we can write:
\[
\SMC_P^{\Gamma,(p)}(Q_S;G)
= \mkern-20mu
\sum_{k:\,\pi_k(S)>0}
\max_{g\in G}
\frac{|B_k(S;g)|^p}{\pi_k(S)^{p-1}},
\qquad
\widehat{\SMC}_T^{\Gamma,(p)}(S;G,\Lambda)
= \mkern-20mu
\sum_{k:\,\widehat\pi_k(S)>0}
\max_{g\in G}
\frac{|\widehat B_k(S;g)|^p}{\widehat\pi_k(S)^{p-1}}.
\]
For each $t\in[T]$ and $k\in[K]$, and each $g\in G$, let
\[
\alpha_{t,k}:=\Ebb[q_{t,k}(X_t)\mid H_{t-1}]
\qquad \text{and} \qquad
b_{t,k}(g):=\Ebb[g(X_t)q_{t,k}(X_t)V(v_k,Y_t)\mid H_{t-1}].
\]
Then
\[
\pi_k(S)=\frac1T\sum_{t=1}^T \alpha_{t,k}
\qquad \text{and} \qquad
B_k(S;g)=\frac1T\sum_{t=1}^T b_{t,k}(g).
\]
Define the martingale differences
\[
N_{t,k}:=q_{t,k}(X_t)-\alpha_{t,k},
\qquad
M_{t,k}(g):=g(X_t)q_{t,k}(X_t)V(v_k,Y_t)-b_{t,k}(g).
\]
Since $q_{t,k}(X_t)\in[0,1]$, we have $|N_{t,k}|\le 1$. Also
$|g(X_t)q_{t,k}(X_t)V(v_k,Y_t)|\le 1,$
so $|M_{t,k}(g)|\le 2$. Moreover,
\[
\Ebb[N_{t,k}^2\mid H_{t-1}]
\le
\Ebb[q_{t,k}(X_t)^2\mid H_{t-1}]
\le
\alpha_{t,k},
\]
\[
\Ebb[M_{t,k}(g)^2\mid H_{t-1}]
\le
\Ebb\bigl[g(X_t)^2q_{t,k}(X_t)^2V(v_k,Y_t)^2\mid H_{t-1}\bigr]
\le
\alpha_{t,k}.
\]
Therefore
\[
\sum_{t=1}^T \Ebb[N_{t,k}^2\mid H_{t-1}]\le T\pi_k(S),
\qquad
\sum_{t=1}^T \Ebb[M_{t,k}(g)^2\mid H_{t-1}]\le T\pi_k(S).
\]

Set $L_0:=3L$. Apply the variance process-adaptive Freedman's inequality (Lemma~\ref{lem:adaptive-freedman}) to each of the $K$ sequences
$(N_{t,k})_{t=1}^T$ with $b=1$, and to each of the $K|G|$ sequences $(M_{t,k}(g))_{t=1}^T$
with $b=2$. Each application of Lemma~\ref{lem:adaptive-freedman} at level $L_0=3L$ fails with probability at most
$2(\lceil\log_2 T\rceil+1)e^{-3L}$. Since there are $K+K|G|\le 2K|G|$ sequences and
$\lceil\log_2 T\rceil+1\le 2T$, a union bound gives total failure probability at most
\[
2K|G|\cdot 2(\lceil\log_2 T\rceil+1)e^{-3L}
\le
\frac1T.
\]
Therefore, there is a constant $C>0$ such that with probability at least $1-1/T$, the following two inequalities hold
simultaneously for all $k\in[K]$ and all $g\in G$ --- which we define as an event $\mathcal{E}$:
\begin{equation}
\label{eq:pi-concentration-unified}
|\widehat\pi_k(S)-\pi_k(S)|
\le
C\left(\sqrt{\frac{\pi_k(S)L}{T}}+\frac{L}{T}\right),
\end{equation}
\begin{equation}
\label{eq:bias-concentration-unified}
|\widehat B_k(S;g)-B_k(S;g)|
\le
C\left(\sqrt{\frac{\pi_k(S)L}{T}}+\frac{L}{T}\right).
\end{equation}

Fix a sufficiently large constant $c_\star>0$ and define the threshold
$\tau:=c_\star\frac{L}{T}$. Now, let us partition the buckets into ``light'' and ``heavy'' ones according to $\tau$:
\[
\mathcal L:=\{k\in[K]:\pi_k(S)<\tau\}
\qquad \text{and} \qquad
\mathcal H:=\{k\in[K]:\pi_k(S)\ge \tau\}.
\]
We will first easily address the light (infrequently used) buckets, and then focus on bounding the empirical-population term deviations for heavy (frequently-used) buckets.

\smallskip
\noindent
\emph{Light buckets.}
For all $k$ and $g\in G$,
$|B_k(S;g)|
\le
\Ebb\bigl[g(X)q_{S,k}(X)|V(v_k,Y)|\bigr]
\le
\Ebb[q_{S,k}(X)]
=
\pi_k(S).$
Hence, for every $k$ with $\pi_k(S)>0$,
$\max_{g\in G}\frac{|B_k(S;g)|^p}{\pi_k(S)^{p-1}}\le \pi_k(S),$
and since the swap objective only sums buckets with positive mass, it follows that
\[
\sum_{k\in\mathcal L:\,\pi_k(S)>0}\max_{g\in G}\frac{|B_k(S;g)|^p}{\pi_k(S)^{p-1}}
\le
\sum_{k\in\mathcal L}\pi_k(S)
\le
K\tau.
\]

\smallskip
\noindent
\emph{Heavy buckets.}
Work on the event $\mathcal E$ and fix $k\in\mathcal H$. If $c_\star$ is chosen large enough,
then \eqref{eq:pi-concentration-unified} implies
$|\widehat\pi_k(S)-\pi_k(S)|\le \frac{\pi_k(S)}2,$
so
$\widehat\pi_k(S)\in\left[\frac{\pi_k(S)}2,\frac{3\pi_k(S)}2\right].$
In particular,
\[
\frac1{\pi_k(S)^{p-1}}
\le
2^{p-1}\frac1{\widehat\pi_k(S)^{p-1}}.
\]
Also, by \eqref{eq:bias-concentration-unified} and $\pi_k(S)\ge \tau=c_\star L/T$,
\[
|B_k(S;g)|
\le
|\widehat B_k(S;g)|
+
C\left(\sqrt{\frac{\pi_k(S)L}{T}}+\frac{L}{T}\right)
\le
|\widehat B_k(S;g)|+C_p\sqrt{\frac{\pi_k(S)L}{T}}.
\]
Using $(u+v)^p\le 2^{p-1}(u^p+v^p)$, we obtain
\[
\frac{|B_k(S;g)|^p}{\pi_k(S)^{p-1}}
\le
2^{p-1}\frac{|\widehat B_k(S;g)|^p}{\pi_k(S)^{p-1}}
+
C_p\left(\frac{L}{T}\right)^{p/2}\pi_k(S)^{1-p/2}.
\]
Combining this with the denominator comparison yields
\[
\max_{g\in G}\frac{|B_k(S;g)|^p}{\pi_k(S)^{p-1}}
\le
C_p\max_{g\in G}\frac{|\widehat B_k(S;g)|^p}{\widehat\pi_k(S)^{p-1}}
+
C_p\left(\frac{L}{T}\right)^{p/2}\pi_k(S)^{1-p/2}.
\]
On the event $\mathcal E$, summing over $k\in\mathcal H$ and adding the light-bucket contribution gives
\begin{equation}
\label{eq:swap-heavy-prebound-clean}
\SMC_P^{\Gamma,(p)}(Q_S;G)
\le
C_p\,\widehat{\SMC}_T^{\Gamma,(p)}(S;G,\Lambda)
+
C_p\sum_{k\in\mathcal H}\left(\frac{L}{T}\right)^{p/2}\pi_k(S)^{1-p/2}
+
K\tau.
\end{equation}
We now explicitly bound the heavy-bucket sum in~\eqref{eq:swap-heavy-prebound-clean}, separately for $p \in [1, 2]$ and for $p > 2$.

If $1\le p<2$, then $a:=1-p/2\in(0,1)$. The concavity of $x\mapsto x^a$ and $\sum_{k=1}^K \pi_k(S)=1$ imply
\[
\sum_{k\in\mathcal H}\pi_k(S)^{1-p/2}
\le
\sum_{k=1}^K \pi_k(S)^a
\le
K^{1-a}
=
K^{p/2},
\]
For the case $p=2$, it also holds that $\sum_{k\in\mathcal H}\pi_k(S)^{1-p/2}
=
\sum_{k\in\mathcal H}1
=
|\mathcal H|
\le
K
=
K^{p/2}.$
Therefore, 
\[
\SMC_P^{\Gamma,(p)}(Q_S;G)
\le
C_p\,\widehat{\SMC}_T^{\Gamma,(p)}(S;G,\Lambda)
+
C_p\left[\left(\frac{KL}{T}\right)^{p/2}+\frac{KL}{T}\right] \text{ on the event } \mathcal{E} \: (\text{for } 1\le p\le 2).
\]
For $p\ge 2$, heavy buckets satisfy $\pi_k(S)\ge \tau$, so
$\pi_k(S)^{1-p/2}\le \tau^{1-p/2}$
and
$\sum\limits_{k\in\mathcal H}\pi_k(S)^{1-p/2}
\le
K\tau^{1-p/2}.$
Thus the heavy-bucket term in \eqref{eq:swap-heavy-prebound-clean} is at most
$C_p K\left(\frac{L}{T}\right)^{p/2}\tau^{1-p/2}
=
C_p\frac{KL}{T}.$
Thus, we obtain
\[
\SMC_P^{\Gamma,(p)}(Q_S;G)
\le
C_p\,\widehat{\SMC}_T^{\Gamma,(p)}(S;G,\Lambda)
+
C_p\frac{KL}{T}
\text{ on the event } \mathcal{E} \:
(\text{for } p\ge 2).
\]

Finally, both $\SMC_P^{\Gamma,(p)}(Q_S;G)$ and
$\widehat{\SMC}_T^{\Gamma,(p)}(S;G,\Lambda)$ lie in $[0,1]$, because every bucket
contribution is at most the corresponding bucket mass and the bucket masses sum to $1$.
Therefore
\[
\Ebb\bigl[\SMC_P^{\Gamma,(p)}(Q_S;G)\bigr]
\le
C_p\,\Ebb\bigl[\widehat{\SMC}_T^{\Gamma,(p)}(S;G,\Lambda)\bigr]
+
\text{transfer term}
+
\Pbb(\mathcal E^c).
\]
Since $\Pbb(\mathcal E^c)\le 1/T$ and $K,L\ge 1$, the final $1/T$ term is absorbed into
$C_p KL/T$. This proves the swap bound. The ordinary bound then follows from
Lemma~\ref{lem:swap-dominates}.
\end{proof}

\subsection{Instantiating the batch swap multicalibration algorithm}
\label{subsec:upper-plug-in}

\begin{proposition}[\cite{hu2025efficient}]
\label{prop:hu-online-input}
Assume the hypotheses of Theorem~\ref{thm:upper-main}. Then there is a grid-supported
online forecaster, i.e.\ a sequence of roundwise distributions
$q_t:\mathcal X\to\Delta(\Lambda)$, such that:

\begin{enumerate}
    \item if $1\le p\le 2$, the forecaster uses a grid of size
    $K=\widetilde\Theta(T^{1/3})$
    and satisfies
    \[
    \Ebb\!\left[\widehat{\SMC}_T^{\Gamma,(p)}(S;G,\Lambda)\right]
    \le
    \widetilde O_{p,\rho}\!\left((\log(2|G|))^{p/2}T^{-p/3}\right);
    \]

    \item if $p\ge 2$, the forecaster uses a grid of size
    $K=\widetilde\Theta(T^{1/(p+1)})$
    and satisfies
    \[
    \Ebb\!\left[\widehat{\SMC}_T^{\Gamma,(p)}(S;G,\Lambda)\right]
    \le
    \widetilde O_{p,\rho}\!\left((\log(2|G|))^{p/2}T^{-p/(p+1)}\right).
    \]
\end{enumerate}
\end{proposition}

\begin{proof}
Hu et al.\ \cite[Corollary~1]{hu2025efficient}, specialized to the finite class $G$ of binary groups, give a
high-probability bound for cumulative swap multicalibration under exactly the displayed
Lipschitz and endpoint assumptions. We encode their deterministic grid prediction
$p_t:\mathcal X\to\Lambda$ in our notation by the point-mass vector
\[
q_{t,k}(x):=\mathbf 1\{p_t(x)=v_k\}.
\]
If the forecaster uses internal randomness, we absorb it into the history as before. With this
choice of $q_t$, our empirical swap objective
$\widehat{\SMC}_T^{\Gamma,(p)}(S;G,\Lambda)$ is exactly the normalized swap
objective of \cite{hu2025efficient}. 
We now recall the bounds that they give for $p \in [1, 2]$ and $p \geq 2$ (coinciding at $p=2$).

For $p\ge 2$, their finite-class construction uses a grid of size
$K=\widetilde\Theta(T^{1/(p+1)})$, and Corollary~1 gives the cumulative swap bound
$\widetilde O_{p,\rho}\!\left((\log(2|G|)+\log(1/\delta))^{p/2}T^{1/(p+1)}\right)$
with probability at least $1-\delta$. Dividing by $T$ yields (on the same event) the normalized empirical bound
\[
\widehat{\SMC}_T^{\Gamma,(p)}(S;G,\Lambda)
\le
\widetilde O_{p,\rho}\!\left((\log(2|G|)+\log(1/\delta))^{p/2}T^{-p/(p+1)}\right).
\]

For $1\le p<2$, Corollary~1 gives the cumulative swap bound
$\widetilde O_{p,\rho}\!\left((\log(2|G|)+\log(1/\delta))^{p/2}T^{1-p/3}\right)$
with probability at least $1-\delta$. In this regime Hu et al.\ obtain the displayed $p<2$
rate from the same finite-class construction with grid size $K=\widetilde\Theta(T^{1/3})$
(via their H\"older reduction from the $p=2$ case). Dividing by $T$ yields (on the same event) the normalized empirical bound
\[
\widehat{\SMC}_T^{\Gamma,(p)}(S;G,\Lambda)
\le
\widetilde O_{p,\rho}\!\left((\log(2|G|)+\log(1/\delta))^{p/2}T^{-p/3}\right).
\]

Finally, the empirical swap error always lies in $[0,1]$, so setting $\delta:=1/T$ and using
$\Ebb[X]\le a+\delta$ whenever $0\le X\le 1$ and $X\le a$ with probability at least
$1-\delta$, we obtain the displayed expectation bound after absorbing the extra
$T^{-1}$ term into the $\widetilde O(\cdot)$ notation.
\end{proof}

\begin{proof}[Proof of Theorem~\ref{thm:upper-main}]
Apply Proposition~\ref{prop:hu-online-input} and then Proposition~\ref{prop:otb-swap-all-p}
to the resulting online forecaster.
If $1\le p\le 2$, then $K=\widetilde\Theta(T^{1/3})$, and the transfer term in
Proposition~\ref{prop:otb-swap-all-p} is therefore
$\widetilde O_p\!\left(\left(\frac{KL}{T}\right)^{p/2}+\frac{KL}{T}\right)
=
\widetilde O_p(T^{-p/3}),$
which is of the same order as the empirical online bound from
Proposition~\ref{prop:hu-online-input}. Hence
\[
\Ebb\bigl[\SMC_P^{\Gamma,(p)}(Q_S;G)\bigr]
\le
\widetilde O_{p,\rho}\!\left((\log(2|G|))^{p/2}T^{-p/3}\right)
\qquad (1\le p\le 2).
\]

If $p\ge 2$, then $K=\widetilde\Theta(T^{1/(p+1)})$, so the transfer term becomes
$\widetilde O_p\!\left(\frac{KL}{T}\right)
=
\widetilde O_p(T^{-p/(p+1)}),$
again of the same order as the empirical online bound. Therefore
\[
\Ebb\bigl[\SMC_P^{\Gamma,(p)}(Q_S;G)\bigr]
\le
\widetilde O_{p,\rho}\!\left((\log(2|G|))^{p/2}T^{-p/(p+1)}\right)
\qquad (p\ge 2).
\]

If the expected error is at most $\varepsilon/3$, the output will
satisfy
$\Pbb\bigl(\SMC_P^{\Gamma,(p)}(Q_S;G)\le \varepsilon\bigr)\ge \frac23.$ by Markov's inequality.
Solving the displayed expectation bounds for $T$ yields exactly the sample-complexity
bounds in the theorem. The ordinary bound follows from Lemma~\ref{lem:swap-dominates}.
\end{proof}

\subsection{Sharper upper bound for mean ECE multicalibration}
\label{subsec:upper-mean-sharp}

\begin{proposition}[ECE online-to-batch reduction]
\label{prop:mean-ece-otb}
Consider the mean property, with the identification function $V(v,y)=v-y$. Then, there is a universal constant $C_{\mathrm{mean}}>0$
such that
\[
\Ebb_{S\sim P^T}\!\left[\MC_P^{\Gamma,(1)}(Q_S;G)\right]
\le
\Ebb_{S\sim P^T}\!\left[\widehat{\MC}_T^{\Gamma,(1)}(S;G,\Lambda)\right]
+
C_{\mathrm{mean}}\sqrt{\frac{K+\log(2|G|)+1}{T}},
\]
where $\Gamma$ is the mean property, so that the ordinary weighted $L_1$ objective is exactly
mean ECE.
\end{proposition}

\begin{proof}
For $\sigma=(\sigma_1,\dots,\sigma_K)\in\{\pm 1\}^K$, define
\[
A^{g,\sigma}_t:=g(X_t)\sum_{k=1}^K \sigma_k q_{t,k}(X_t)(v_k-Y_t).
\]
Also define the population and empirical signed scores
\[
a_{g,\sigma}(S):=\sum_{k=1}^K \sigma_k\,\Ebb_{(X,Y)\sim P}[g(X)q_{S,k}(X)(v_k-Y)]
\qquad \text{and} \qquad
\widehat a_{g,\sigma}(S):=\frac1T\sum_{t=1}^T A^{g,\sigma}_t.
\]
Using the definition of $q_{S,k}$ and then averaging over $t$,
\[
a_{g,\sigma}(S)
=
\frac1T\sum_{t=1}^T
\Ebb_{(X,Y)\sim P}\!\left[g(X)\sum_{k=1}^K \sigma_k q_{t,k}(X)(v_k-Y)\right].
\]
Since $q_t$ is $H_{t-1}$-measurable and $(X_t,Y_t)$ is a fresh draw from $P$ independent of
$H_{t-1}$,
\[
\Ebb[A^{g,\sigma}_t\mid H_{t-1}]
=
\Ebb_{(X,Y)\sim P}\!\left[g(X)\sum_{k=1}^K \sigma_k q_{t,k}(X)(v_k-Y)\right] \text{ almost surely.}
\]
Thus, letting $(M^{g,\sigma}_t)_{t=1}^T$ be a martingale difference sequence defined as $M^{g,\sigma}_t:=\Ebb[A^{g,\sigma}_t\mid H_{t-1}]-A^{g,\sigma}_t$,
\[
a_{g,\sigma}(S)-\widehat a_{g,\sigma}(S)=\frac1T\sum_{t=1}^T M^{g,\sigma}_t.
\]
Because $g(X_t)\in\{0,1\}$, the coefficients $q_{t,k}(X_t)$ form a probability vector, and each
$v_k-Y_t\in[-1,1]$, we have $A^{g,\sigma}_t\in[-1,1]$ and hence
$|M^{g,\sigma}_t|\le 2.$
Thus, by Azuma--Hoeffding, for every $(g,\sigma)$ and $\eta>0$,
\[
\Pbb\bigl(a_{g,\sigma}(S)-\widehat a_{g,\sigma}(S)\ge \eta\bigr)
\le
\exp\!\left(-\frac{T\eta^2}{8}\right).
\]
Now define (with the second equality using that $-\sigma\in\{\pm 1\}^K$ whenever $\sigma\in\{\pm 1\}^K$):
\[
\Xi(S):=\max_{g\in G}\max_{\sigma\in\{\pm 1\}^K}
\bigl(a_{g,\sigma}(S)-\widehat a_{g,\sigma}(S)\bigr) 
=\max_{g\in G}\max_{\sigma\in\{\pm 1\}^K}
\left|a_{g,\sigma}(S)-\widehat a_{g,\sigma}(S)\right|.
\]
For every real vector $(b_1,\dots,b_K)$,
\[
\sum_{k=1}^K |b_k| = \max_{\sigma\in\{\pm 1\}^K}\sum_{k=1}^K \sigma_k b_k.
\]
Applying this identity bucketwise shows that
\[
\MC_P^{\Gamma,(1)}(Q_S;G)
\le
\widehat{\MC}_T^{\Gamma,(1)}(S;G,\Lambda)+\Xi(S).
\]
Thus, it suffices to bound $\Ebb[\Xi(S)]$. First, since there are $|G|2^K$ pairs $(g,\sigma)$, a union bound gives
\[
\Pbb\bigl(\Xi(S)\ge \eta\bigr)
\le
|G|2^K\exp\!\left(-\frac{T\eta^2}{8}\right)
\qquad (\eta>0).
\]
Now, let
$N:=\max\{e,|G|2^K\},
\:
\eta_0:=\sqrt{\frac{8\log N}{T}}.$
Since $\Xi(S)\ge 0$,
\[
\Ebb[\Xi(S)]
=
\int_0^\infty \Pbb\bigl(\Xi(S)\ge \eta\bigr)\,d\eta
\le
\eta_0+
\int_{\eta_0}^\infty N\exp\!\left(-\frac{T\eta^2}{8}\right)d\eta.
\]
Using the Gaussian tail bound
$\int_u^\infty e^{-s^2}\,ds\le \frac{e^{-u^2}}{2u}
\: (u>0),$ and that $N \geq e$,
we can see that the second term is at most $4N(T\eta_0)^{-1}e^{-T\eta_0^2/8}\le \eta_0$. Hence, for a universal constant $C_{\mathrm{mean}}>0$,
\[
\Ebb[\Xi(S)]\le 2\eta_0\le C_{\mathrm{mean}}\sqrt{\frac{K+\log(2|G|)+1}{T}}. \qedhere
\]
\end{proof}

\begin{lemma}[Rounding interval buckets to a prediction grid]
\label{lem:bucket-rounding}
Fix an integer $K\ge 1$, let
\[
I_k:=
\begin{cases}
\left[\dfrac{k-1}{K},\dfrac{k}{K}\right), & k=1,\dots,K-1,\\[1.2ex]
\left[\dfrac{K-1}{K},1\right], & k=K,
\end{cases}
\qquad
v_k:=\frac{2k-1}{2K}.
\]
For each round $t$, let $\bar p_t$ be a probability distribution on $[0,1]$, and define
\[
q_{t,k}:=\Pbb_{P\sim \bar p_t}[P\in I_k].
\]
Then for every choice of labels $y_1,\dots,y_T\in[0,1]$ and weights $g_1,\dots,g_T\in\{0,1\}$,
\[
\sum_{k=1}^K \left|\frac1T\sum_{t=1}^T g_t q_{t,k}(v_k-y_t)\right|
\le
\frac{1}{2K}
+
\sum_{k=1}^K
\left|\frac1T\sum_{t=1}^T \Ebb_{P\sim \bar p_t}[g_t\mathbf 1\{P\in I_k\}(P-y_t)]\right|.
\]
\end{lemma}

\begin{proof}
For each $k$,
\[
\frac1T\sum_{t=1}^T g_t q_{t,k}(v_k-y_t)
=
\frac1T\sum_{t=1}^T \Ebb_{P\sim \bar p_t}[g_t\mathbf 1\{P\in I_k\}(P-y_t)]
+
\frac1T\sum_{t=1}^T \Ebb_{P\sim \bar p_t}[g_t\mathbf 1\{P\in I_k\}(v_k-P)].
\]
Summing the triangle inequality over $k$ therefore yields
\[
\sum_{k=1}^K \left|\frac1T\sum_{t=1}^T g_t q_{t,k}(v_k-y_t)\right|
\le
\sum_{k=1}^K
\left|\frac1T\sum_{t=1}^T \Ebb_{P\sim \bar p_t}[g_t\mathbf 1\{P\in I_k\}(P-y_t)]\right|
+
\sum_{k=1}^K |D_k|,
\]
where
\[
D_k:=\frac1T\sum_{t=1}^T \Ebb_{P\sim \bar p_t}[g_t\mathbf 1\{P\in I_k\}(v_k-P)].
\]
If $P\in I_k$, then $|v_k-P|\le 1/(2K)$. Hence we obtain the claim by noting that
\[
\sum_{k=1}^K |D_k|
\le
\frac{1}{2KT}\sum_{t=1}^T g_t \sum_{k=1}^K \Pbb_{P\sim \bar p_t}[P\in I_k]
=
\frac{1}{2KT}\sum_{t=1}^T g_t
\le
\frac{1}{2K}. \qedhere
\]
\end{proof}

\begin{proof}[Proof of Theorem~\ref{cor:mean-ece-sharper}]
Let
$K:=\max\!\left\{1,\left\lceil\left(\frac{T}{\log(2|G|T)}\right)^{1/3}\right\rceil\right\},
\:
\Lambda:=\left\{\frac{2k-1}{2K}:k\in[K]\right\}.$
Run the online algorithm of \cite{noarov2025high} on the i.i.d.\ transcript
$S=((X_1,Y_1),\dots,(X_T,Y_T))\sim P^T.$
At round $t$, after observing the history $H_{t-1}$, let
$\bar p_t=\bar p_t^{H_{t-1}}:\mathcal X\to\Delta([0,1])$
denote the resulting randomized prediction rule. 

For each interval bucket $I_k$ from
Lemma~\ref{lem:bucket-rounding}, define the rounded grid weights by
$q_{t,k}(x):=\Pbb_{P\sim \bar p_t(x)}[P\in I_k].$
Since the intervals $(I_k)_{k=1}^K$ partition $[0,1]$, this defines a roundwise distribution-valued
rule $q_t:\mathcal X\to\Delta(\Lambda)$ and therefore an averaged batch predictor $Q_S$ as in
Subsection~\ref{subsec:upper-setup}.

Fix $g\in G$ and $k\in[K]$, and consider the bucket event
\[
E_{g,k}(x,p):=g(x)\mathbf 1\{p\in I_k\}.
\]
Define the incidence of the bucket event as
$n_{g,k}(S):=\sum_{t=1}^T \Ebb_{P\sim \bar p_t(X_t)}[E_{g,k}(X_t,P)].$
Note that $\sum_{k=1}^K n_{g,k}(S)
=
\sum_{t=1}^T g(X_t) \sum_{k=1}^K q_{t,k}(X_t)
\le T$, since $\sum_{k=1}^K q_{t,k}(X_t)=1$ and $g(X_t)\in\{0,1\}$.

We use Noarov et al.\ \cite[Theorem~2.4]{noarov2025high} in the pathwise bias-guarantee form, so the
following inequality holds for every realized transcript $S$. Applied to the event family
$\{E_{g,k}:g\in G,\ k\in[K]\}$, it gives simultaneously for all $g\in G$ and $k\in[K]$ the
cumulative bias bound
\[
\left|
\sum_{t=1}^T
\Ebb_{P\sim \bar p_t(X_t)}[g(X_t)\mathbf 1\{P\in I_k\}(P-Y_t)]
\right|
\le
C_N\left(L+\sqrt{n_{g,k}(S)L}\right)
\]
where $L:=\log(2K|G|T)$ and $C_N>0$ is a constant.
Dividing this cumulative bound by $T$, summing over $k$, and using Cauchy--Schwarz with $\sum_k n_{g,k}(S)\le T$ gives:
\[
\sum_{k=1}^K
\left|\frac1T\sum_{t=1}^T
\Ebb_{P\sim \bar p_t(X_t)}[g(X_t)\mathbf 1\{P\in I_k\}(P-Y_t)]
\right|
\le
C_N\frac{KL}{T}+C_N\sqrt{\frac{KL}{T}}.
\]
Now, to round the interval predictions to a prediction grid, we apply Lemma~\ref{lem:bucket-rounding} with $g_t:=g(X_t)$ and $y_t:=Y_t$ and take the
maximum over $g\in G$, yielding the empirical mean-ECE bound
\[
\widehat{\MC}_T^{\Gamma,(1)}(S;G,\Lambda)
\le
\tfrac{1}{2K}+C_N\tfrac{KL}{T}+C_N\sqrt{\tfrac{KL}{T}}.
\]
Taking expectations and applying the online-to-batch reduction Proposition~\ref{prop:mean-ece-otb}, we then obtain (substituting the choice of $K$ from above) the bound:
\[
\Ebb\bigl[\MC_P^{\Gamma,(1)}(Q_S;G)\bigr]
\le
\tfrac{1}{2K}+C_N\tfrac{KL}{T}+C_N\sqrt{\tfrac{KL}{T}}
+C_{\mathrm{mean}}\sqrt{\tfrac{K+\log(2|G|)+1}{T}}
= \widetilde O\!\left(\left(\tfrac{\log(2|G|T)}{T}\right)^{1/3}\right).
\]
Since the objective is always in $[0,1]$, Markov's inequality converts this expectation bound into
success probability at least $2/3$ after adjusting constants. Finally, solving for $T$ yields
\[
T=\widetilde O\!\left(\varepsilon^{-3}\log(2|G|)\right). \qedhere
\]
\end{proof}

\section{Discussion and Open Problems}
An interesting question that our paper leaves open is whether randomization is \emph{necessary} to achieve the minimax optimal sample complexity rates that we derive. In the notation of Definition~\ref{def:sample-complexity}, this asks whether restricting the infimum in $\SC_{\mathrm{mean\text{-}ECE}}^{(\kappa)}(\varepsilon)$, $\SC_{L_p}^{(\kappa)}(\varepsilon)$, and their property-specific analogues to learners that output \emph{deterministic} models changes the asymptotic rate. One natural route to derandomization is to (randomly) fix the model's internal randomness in advance and appeal to concentration: if no single context $x$ carries substantial probability mass, then a random realization of the predictor should preserve its calibration guarantees on most of the distribution. This heuristic breaks down in the presence of atoms of nontrivial mass, where a single unlucky realization at one context can create a large calibration error. It would be interesting to understand whether deterministic predictors can match the randomized minimax rates in full generality, or whether there is a genuine gap between deterministic and randomized multicalibration.

Our results characterize the minimax dependence on $\varepsilon$ for mean ECE and weighted $L_p$ multicalibration throughout the entire regime of polynomial group growth, and, under the corresponding upper-bound hypotheses, the same exponents extend to relevant concrete property classes such as quantiles and expectiles. The point is that our lower-bound constructions themselves use only polylogarithmically many groups: after choosing the hard-instance parameter as a function of $\varepsilon$, we have $|G|=\mathrm{polylog}(1/\varepsilon)$. Consequently the sharp lower bounds $\widetilde{\Omega}(\varepsilon^{-3})$ for mean ECE and $\widetilde{\Omega}(\varepsilon^{-3/p})$ for weighted $L_p$ already fit every polynomial budget $|G|\le \varepsilon^{-\kappa}$ with fixed $\kappa>0$, and therefore match the corresponding upper bounds up to polylogarithmic factors.

Some growth in $|G|$ is nevertheless essential. For a fixed binary group family $G$, the groups induce a partition of the domain into at most $2^{|G|}$ parts, corresponding to the possible intersection patterns of the groups. Estimating the mean on each cell and predicting that cellwise empirical mean yields the trivial upper bound $\widetilde{O}(2^{|G|}/\varepsilon^2)$ for mean ECE. Thus one cannot hope for a uniform $\widetilde{\Omega}(\varepsilon^{-3})$ lower bound when $|G|$ is fixed, and at the exponent level the constant-budget regime behaves like ordinary mean estimation.

The remaining open problem is therefore to determine the optimal joint dependence of the minimax sample complexity on both $\varepsilon$ and the group-budget parameter $M$. In particular, it would be very interesting to know whether the current $\widetilde{O}((\log |G|)/\varepsilon^3)$ upper bounds for mean ECE can be improved to $\widetilde{O}(f(|G|)/\varepsilon^3)$ for some $f(x)=o(\log x)$, or whether such sublogarithmic dependence can be ruled out in regimes where $|G|$ grows with $1/\varepsilon$ but remains subpolynomial. More generally, one can ask for the sharp interpolation between the trivial $2^{|G|}/\varepsilon^2$ upper bound for small $|G|$ and the $\varepsilon^{-3/p}$ behavior that our lower bounds show is unavoidable throughout the polynomial-budget regime, as well as for analogous statements in terms of structural complexity parameters such as VC dimension.

\subsection*{Acknowledgments}

The authors used AI tools, specifically GPT 5.4 Pro, and GPT 5.4 in the Codex environment in the development of this paper; all of the final theorems and proofs are written and verified by the authors, and all of  the exposition and discussion of related work was written without AI assistance.  

\bibliographystyle{alpha}
\bibliography{refs}

\appendix

\section{Regularity Verifications}
\label{app:regular-verifications}

\subsection{The mean property}

\begin{proposition}[The mean admits a Bernoulli regularity witness]
\label{prop:mean-regular-app}
Let $\Gamma_{\mathrm{mean}}$ be the mean property,
\[
V_{\mathrm{mean}}(v,y):=v-y,
\]
and $D_t:=\Ber(t)$. Then for every closed interval $I_0\subset(0,1)$, the
family $(\Gamma_{\mathrm{mean}},V_{\mathrm{mean}},\{D_t:t\in I_0\})$ is a
regularity witness on $I_0$. Moreover, on $I_0=[1/4,3/4]$ one may take
\[
c_\Gamma=1,
\qquad
C_{\Gamma,\mathrm{KL}}=\frac{16}{3}.
\]
\end{proposition}

\begin{proof}
For $Y\sim \Ber(t)$,
\[
M_\Gamma(v,t)=\Ebb[v-Y]=v-t,
\]
so
\[
(v-t)M_\Gamma(v,t)=(v-t)^2.
\]
Thus condition~(2) in Definition~\ref{def:regular} holds with $c_\Gamma=1$, and
condition~(1) is immediate.

Fix any closed interval $I_0\subset(0,1)$ and let
\[
\delta_{I_0}:=\min_{u\in I_0} u(1-u)>0.
\]
If $p,q\in I_0$, then by $\log x\le x-1$,
\begin{align*}
\KL(\Ber(p)\,\|\,\Ber(q))
&=p\log\frac{p}{q}+(1-p)\log\frac{1-p}{1-q}\\
&\le p\left(\frac{p}{q}-1\right)+(1-p)\left(\frac{1-p}{1-q}-1\right)\\
&=\frac{(p-q)^2}{q(1-q)}
\le \frac{(p-q)^2}{\delta_{I_0}}.
\end{align*}
This proves condition~(3) on $I_0$. On $I_0=[1/4,3/4]$, the same calculation
sharpens to
\[
\KL(\Ber(p)\,\|\,\Ber(q))\le \frac{(p-q)^2}{(1/4)(3/4)}=\frac{16}{3}(p-q)^2,
\]
which is the claimed constant.
\end{proof}

\subsection{Expectiles}

\begin{definition}[$\tau$-expectile]
Fix an expectile level $\tau\in(0,1)$. For a distribution $\nu$ on $[0,1]$, the $\tau$-expectile $\Gamma_\tau(\nu)$ is the unique value $v\in[0,1]$ satisfying
\[
\E_{Y\sim \nu}\!\left[\mathsf{V}_\tau(v,Y)\right]=0,
\qquad
\mathsf{V}_\tau(v,y):=
\left|\tau-\ind{y\le v}\right|(v-y).
\]
\end{definition}

\begin{proposition}[Expectiles admit a regularity witness]
\label{prop:expectile-regular-app}
Fix $\tau\in(0,1)$. For $t\in[1/4,3/4]$, define
\[
p_\tau(t):=\frac{(1-\tau)t}{\tau+(1-2\tau)t}
\qquad\text{and}\qquad
D_t:=\Ber\bigl(p_\tau(t)\bigr).
\]
Then $(\Gamma_\tau,V_\tau,\{D_t:t\in[1/4,3/4]\})$ is a regularity witness. Moreover, for every
distribution $\nu$ on $[0,1]$, the map
\[
v\mapsto \Ebb_{Y\sim \nu}[V_\tau(v,Y)]
\]
is $1$-Lipschitz on $[0,1]$.
\end{proposition}

\begin{proof}
For $v\in[0,1]$ and $Y\sim\Ber(p)$,
\[
\Ebb[V_\tau(v,Y)]=(1-p)(1-\tau)v+p\tau(v-1).
\]
Substituting $p=p_\tau(t)$ gives
\[
\Ebb_{Y\sim D_t}[V_\tau(v,Y)]
=
a_{\tau,t}(v-t),
\qquad
a_{\tau,t}:=\frac{\tau(1-\tau)}{\tau+(1-2\tau)t}.
\]
Hence $\Gamma_\tau(D_t)=t$ and
\[
(v-t)\Ebb_{Y\sim D_t}[V_\tau(v,Y)]
=
a_{\tau,t}(v-t)^2.
\]
Since
\[
\tau+(1-2\tau)t\in[\min\{\tau,1-\tau\},\max\{\tau,1-\tau\}]
\qquad (t\in[0,1]),
\]
we have
\[
a_{\tau,t}\ge \min\{\tau,1-\tau\}>0.
\]
So the quantitative sign condition holds with
\[
c_{\Gamma_\tau}:=\min\{\tau,1-\tau\}.
\]

For the KL condition, note that $p_\tau$ is smooth on $[0,1]$ with derivative
\[
p_\tau'(t)=\frac{\tau(1-\tau)}{(\tau+(1-2\tau)t)^2}.
\]
Thus $p_\tau$ is Lipschitz on $[1/4,3/4]$; let
\[
L_\tau:=\sup_{t\in[1/4,3/4]}|p_\tau'(t)|.
\]
Also $p_\tau([1/4,3/4])$ is a compact subset of $(0,1)$, so there is $\eta_\tau>0$ such that
\[
p_\tau(t)\in[\eta_\tau,1-\eta_\tau]
\qquad (t\in[1/4,3/4]).
\]
For $t,t'\in[1/4,3/4]$,
\[
\KL(D_t\,\|\,D_{t'})
\le
\frac{(p_\tau(t)-p_\tau(t'))^2}{\eta_\tau(1-\eta_\tau)}
\le
\frac{L_\tau^2}{\eta_\tau(1-\eta_\tau)}(t-t')^2.
\]
So the KL regularity condition holds as well. For the final claim, fix $y\in[0,1]$. The function
\[
v\mapsto V_\tau(v,y)
\]
is continuous and piecewise linear, with slope $\tau$ on $(-\infty,y)$ and slope $1-\tau$ on
$(y,\infty)$. Hence it is $1$-Lipschitz on $[0,1]$. Taking expectation over $Y\sim \nu$ preserves
the same Lipschitz constant.
\end{proof}

\subsection{Quantiles}

\begin{definition}[$q$-quantile]
Fix a quantile level $q\in(0,1)$. Consider a distribution $\nu$ on $[0,1]$, and let $F_\nu$ be the CDF of $\nu$. Then, the $q$-quantile $\Gamma_q(\nu)$ of this distribution is
\[
\Gamma_q(\nu):=\inf\{v\in[0,1]:F_\nu(v)\ge q\}.
\]
We use the standard identification function
\[
\mathsf{V}_q(v,y):=\ind{y\le v}-q.
\]
\end{definition}

For $\lambda\in\R$, let $D_\lambda$ be the truncated exponential distribution on $[0,1]$ with density
\[
f_\lambda(y):=\exp(\lambda y-A(\lambda)),
\qquad
A(\lambda):=\log\!\int_0^1 e^{\lambda z}\,dz.
\]

\begin{proposition}[Quantiles admit a regularity witness]
\label{prop:quantile-regular-app}
For every $q\in(0,1)$ there exist a constant $\Lambda_q>0$ and a closed interval $I_q\subset(0,1)$
such that, after reparameterizing the family $\{D_\lambda:\lambda\in[-\Lambda_q,\Lambda_q]\}$ by its
$q$-quantile, $(\Gamma_q,V_q,\{D_t:t\in I_q\})$ is a regularity witness. Moreover, the family
$\{D_t:t\in I_q\}$ has densities uniformly bounded above by some constant $C_q<\infty$.
\end{proposition}

\begin{proof}
For $\lambda\neq 0$, the CDF of $D_\lambda$ is
\[
F_\lambda(v)=\frac{e^{\lambda v}-1}{e^\lambda-1},
\qquad v\in[0,1],
\]
and for $\lambda=0$ it is the uniform CDF $F_0(v)=v$. Hence the $q$-quantile of $D_\lambda$ is
\[
t_q(\lambda):=
\begin{cases}
\dfrac{1}{\lambda}\log\bigl(1+q(e^\lambda-1)\bigr),&\lambda\neq 0,\\[1ex]
q,&\lambda=0.
\end{cases}
\]
This function is smooth and satisfies
\[
t_q'(0)=\frac{q(1-q)}{2}>0.
\]
By continuity, there is $\Lambda_q>0$ and constants $0<\ell_q\le L_q<\infty$ such that
\[
\ell_q\le t_q'(\lambda)\le L_q
\qquad\text{for every }\lambda\in[-\Lambda_q,\Lambda_q].
\]
Shrinking $\Lambda_q$ further if necessary, we may also ensure that
\[
t_q([-\Lambda_q,\Lambda_q])\subset(0,1).
\]
Therefore $t_q$ is bi-Lipschitz on that interval. Let
\[
I_q:=t_q([-\Lambda_q,\Lambda_q])
\]
and let $\lambda_q:I_q\to[-\Lambda_q,\Lambda_q]$ be the inverse map. For $t\in I_q$, define
\[
D_t:=D_{\lambda_q(t)}.
\]
By construction, $\Gamma_q(D_t)=t$.

Now fix $t\in I_q$ and $v\in[0,1]$. Since
\[
\Ebb_{Y\sim D_t}[V_q(v,Y)]
=
F_{\lambda_q(t)}(v)-q
=
F_{\lambda_q(t)}(v)-F_{\lambda_q(t)}(t),
\]
the mean value theorem gives
\[
\Ebb_{Y\sim D_t}[V_q(v,Y)]
=
f_{\lambda_q(t)}(\xi)(v-t)
\]
for some $\xi$ between $v$ and $t$. The function $(\lambda,y)\mapsto f_\lambda(y)$ is continuous
and strictly positive on the compact set $[-\Lambda_q,\Lambda_q]\times[0,1]$, so there are constants
$0<c_q\le C_q<\infty$ such that
\[
c_q\le f_\lambda(y)\le C_q
\qquad (\lambda\in[-\Lambda_q,\Lambda_q],\ y\in[0,1]).
\]
Hence
\[
c_q(v-t)^2\le (v-t)\Ebb_{Y\sim D_t}[V_q(v,Y)]\le C_q(v-t)^2,
\]
so the sign condition holds with $c_{\Gamma_q}=c_q$.

For the KL condition, note that $D_\lambda$ is a one-dimensional exponential family. Therefore
\[
\KL(D_\lambda\,\|\,D_{\lambda'})=A(\lambda')-A(\lambda)-(\lambda'-\lambda)A'(\lambda).
\]
Taylor's theorem gives
\[
\KL(D_\lambda\,\|\,D_{\lambda'})=\frac{A''(\xi)}{2}(\lambda-\lambda')^2
\]
for some $\xi$ between $\lambda$ and $\lambda'$. Since
\[
A''(\xi)=\mathrm{Var}_{D_\xi}(Y)\le \frac14,
\]
we obtain
\[
\KL(D_\lambda\,\|\,D_{\lambda'})\le \frac18(\lambda-\lambda')^2.
\]
Finally, the inverse map $\lambda_q$ is Lipschitz on $I_q$ with constant $1/\ell_q$, so
\[
|\lambda_q(t)-\lambda_q(t')|\le \frac{1}{\ell_q}|t-t'|.
\]
Therefore
\[
\KL(D_t\,\|\,D_{t'})\le \frac{1}{8\ell_q^2}(t-t')^2,
\]
which proves that the displayed family is a regularity witness. The same compactness argument already gave the uniform density upper
bound $C_q$.
\end{proof}

\section{Details for the related-work rate conversions}
\label{app:rw-rate-conversions}

This appendix records the calculations behind the converted upper-bound rates summarized in Table~\ref{tab:intro-rates}. We translate each prior guarantee into the mean ECE metric used in the main body. Throughout this section, $G$ is a finite family of binary groups, confidence is a fixed constant, and $\widetilde O(\cdot)$ hides polylogarithmic factors in $1/\varepsilon$, $|G|$, and any confidence parameter. When a paper states a bound in terms of $\mathrm{VC}(G)$, we use the standard inequality $\mathrm{VC}(G)\le \log_2 |G|$ for finite classes.

The conversions repeatedly use the same bucket-to-ECE calculation. Suppose a predictor uses $\lambda$ prediction buckets and has per-bucket bias at most $\alpha$ on every group. Rounding each prediction to the center of its bucket changes the bias on any example by at most $O(1/\lambda)$, and the sum of the absolute bucket biases is at most $\lambda \alpha$. Thus the resulting mean ECE is at most
\[
\lambda \alpha + O(1/\lambda).
\]
This is the same center-rounding calculation formalized in Lemma~\ref{lem:bucket-rounding}.

\paragraph{\cite{hebert2018multicalibration}}
Definition~2 of \cite{hebert2018multicalibration} defines $\alpha$-calibration on a group after discarding an $\alpha$ fraction of the mass, and Theorem~2 proves that when every group in $G$ has mass at least $\gamma$, running their algorithm with $\lambda=\alpha$ returns a $(G,2\alpha)$-multicalibrated predictor from
\[
\widetilde O\!\left(\frac{\log |G|}{\alpha^{11/2}\gamma^{3/2}}\right)
\]
samples. Immediately after Definition~2 they note that $\beta$-calibration implies $2\beta$ accurate-in-expectation error. Applying this with $\beta=2\alpha$, their theorem yields mean ECE of order $\alpha$ on every group of mass at least $\gamma$. For a group of mass below $\gamma$, our ECE metric is trivially at most $\gamma$ because $|v-Y|\le 1$. Hence the unrestricted mean ECE is at most $O(\alpha+\gamma)$. Choosing $\alpha\asymp \gamma\asymp \varepsilon$ gives
\[
\widetilde O\!\left(\frac{\log |G|}{\varepsilon^{7}}\right).
\]
Under their original minimum-group-mass assumption, one may instead regard $\gamma$ as a fixed parameter and recover the sharper conditional rate
\[
\widetilde O\!\left(\frac{\log |G|}{\varepsilon^{11/2}\gamma^{3/2}}\right).
\]

\paragraph{\cite{gupta2021online}}
Definition~2 of \cite{gupta2021online} is bucketed $(\alpha,n)$-mean multicalibration: predictions are grouped into $n$ buckets of width $1/n$, and every group-bucket pair must have mean bias at most $\alpha$. Their Theorem~6 gives online bucketed mean multicalibration with
\[
\alpha = \widetilde O\!\left(\sqrt{\frac{\log(|G|n)}{T}}\right),
\]
and their Appendix~A gives the corresponding randomized batch predictor; specifically, Theorem~A.1 states the same $\widetilde O(\sqrt{\log(|G|n)/T})$ bucketed guarantee in the batch setting. Rounding every bucket to its center therefore gives mean ECE at most
\[
\widetilde O\!\left(n\sqrt{\frac{\log(|G|n)}{T}}+\frac{1}{n}\right).
\]
Setting $n\asymp 1/\varepsilon$ and requiring each term to be $O(\varepsilon)$ gives
\[
T = \widetilde O\!\left(\frac{\log |G|}{\varepsilon^4}\right).
\]

\paragraph{Haghtalab, Jordan, and Zhao.}
Definition~2.1 of \cite{haghtalab2023unifying} is the standard $\lambda$-bucket multicalibration condition: for each group, bucket, and class coordinate, the bucket bias is at most $\varepsilon$ in absolute value. In the binary mean setting, Theorem~4.1 returns a randomized predictor that is $(G,\varepsilon,\lambda)$-multicalibrated using
\[
O\!\left(\frac{\log |G|+\log \lambda}{\varepsilon^2}\right)
\]
samples. Rounding every bucket to its center gives mean ECE at most
\[
\lambda\varepsilon + O(1/\lambda).
\]
Choosing $\lambda\asymp 1/\eta$ and $\varepsilon\asymp \eta^2$ to target mean ECE $\eta$ yields
\[
\widetilde O\!\left(\frac{\log |G|}{\eta^4}\right)
\]
samples.

\paragraph{\cite{gopalan2022low} (Low-Degree Multicalibration).}
\cite{gopalan2022low} study the generalized condition
\[
\left|\E[c(X)\,w(f(X))(Y-f(X))]\right|\le \alpha
\]
over a hypothesis class $C$ and a weight class $W$. For fixed degree $k$, the low-degree part of Theorem~35 gives sample complexity
\[
\widetilde O\!\left(\frac{\log |G|}{\alpha^4}\right)
\]
for degree-$(k+1)$ multicalibration in the binary setting. To recover a rate for our mean ECE metric, specialize instead to their full multicalibration notion, which uses the interval basis $I_\delta$ of width $\delta$. Proposition~34 with $|I_\delta|=O(1/\delta)$ gives sample complexity
\[
\widetilde O\!\left(\frac{\log |G|}{\alpha^4}\right)
\]
for $(G,I_\delta,\alpha)$-multicalibration. In the binary case this implies mean ECE at most
\[
O(\alpha/\delta+\delta),
\]
because there are $O(1/\delta)$ interval buckets and center-rounding contributes another $O(\delta)$ term. Choosing $\delta\asymp \varepsilon$ and $\alpha\asymp \varepsilon^2$ gives
\[
\widetilde O\!\left(\frac{\log |G|}{\varepsilon^8}\right).
\]

\paragraph{\cite{globus2023boosting}}
Definition~2.1 of \cite{globus2023boosting} is the exact weighted $L_2$ multicalibration quantity
\[
K^2(f,h,D)
:=
\sum_{v\in R(f)} \Pr[f(X)=v]\,
\E[h(X)(Y-v)\mid f(X)=v]^2.
\]
Their Theorem~4.3 gives an in-sample boosting algorithm, and Theorem~C.9 gives the out-of-sample guarantee: with
\[
n=\widetilde O\!\left(\frac{dB^5}{\alpha^5}\right)
\]
samples, where $d=\mathrm{Pdim}(H)$, the output predictor is $O(\alpha)$-approximately multicalibrated in this weighted $L_2$ sense. For finite classes, $d=O(\log |G|)$. By Cauchy--Schwarz,
\[
\sum_v \pi_v |\mathrm{bias}_v|
\le
\left(\sum_v \pi_v \,\mathrm{bias}_v^2\right)^{1/2},
\]
so mean ECE is at most the square root of their weighted $L_2$ error. Therefore to achieve mean ECE at most $\varepsilon$, it suffices to take $\alpha\asymp \varepsilon^2$, which yields
\[
\widetilde O\!\left(\frac{\log |G|}{\varepsilon^{10}}\right).
\]
Equivalently, their own weighted $L_2$ metric has sample complexity $\widetilde O(\rho^{-5}\log |G|)$ at target error $\rho$.

\paragraph{\cite{noarov2025high}}
Theorem~2.4 of \cite{noarov2025high} gives a per-event cumulative bias bound of order
\[
\widetilde O\!\left(\sqrt{n_T(E)}\right)
\]
simultaneously for all conditioning events $E$. Their Section~2.2 applies this to mean multicalibration with $m$ prediction buckets and obtains cumulative bucketed multicalibration error
\[
\widetilde O\!\left(\frac{T}{m}+\sqrt{Tm}\right).
\]
Dividing by $T$ gives normalized bucketed error
\[
\widetilde O\!\left(\frac{1}{m}+\sqrt{\frac{m}{T}}\right).
\]
Choosing $m\asymp T^{1/3}$ yields normalized bucketed error $\widetilde O(T^{-1/3})$. Lemma~\ref{lem:bucket-rounding} converts this to the same ECE rate on the transcript, and our online-to-batch conversion then yields the same order in batch. Hence their online theorem implies batch mean ECE sample complexity
\[
\widetilde O\!\left(\frac{\log |G|}{\varepsilon^3}\right),
\]
which is the sharp upper bound used in the main text.

\section{Quantizing arbitrary randomized predictors}

The main text works with finitely supported randomized predictors because that is the form produced by the online-to-batch reduction and because it lets us write ECE as an explicit sum over prediction values. For lower-bound purposes this is essentially without loss of generality: an arbitrary randomized predictor can be quantized to a finite grid with arbitrarily small additive loss relative to the total-variation extension below.

Let $P$ be a distribution on $\mathcal{X}\times[0,1]$, let $\mu(x):=\E[Y\mid X=x]$, and let $Q=(Q_x)_{x\in\mathcal{X}}$ be an arbitrary randomized predictor. For a signed weight $w:\mathcal{X}\to[-1,1]$, define the signed measure
\[
\nu_{P,Q;w}(A):=\E\!\left[w(X)\int_A (v-Y)\,Q_X(dv)\right]
\]
for Borel sets $A\subseteq [0,1]$, and set
\[
\Err_P^{\mathrm{TV}}(Q;w):=\|\nu_{P,Q;w}\|_{\mathrm{TV}},\qquad
\MC_P^{\mathrm{TV}}(Q;G):=\max_{g\in G}\Err_P^{\mathrm{TV}}(Q;g),
\]
\[
\Delta_P^{\mathrm{TV}}(Q):=\E\!\left[\int |v-\mu(X)|\,Q_X(dv)\right].
\]
If $Q$ is finitely supported, then $\nu_{P,Q;w}=\sum_{v\in V(Q)} B_P(Q;v,w)\,\delta_v$, so $\Err_P^{\mathrm{TV}}(Q;w)=\Err_P(Q;w)$ and $\MC_P^{\mathrm{TV}}(Q;G)=\MC_P(Q;G)$.

\begin{proposition}[Quantization of arbitrary randomized predictors]
\label{prop:quantize-arbitrary}
For every $\eta>0$ there exists a map $T_\eta:[0,1]\to[0,1]$ with finite image and $|T_\eta(v)-v|\le \eta$ for all $v\in[0,1]$. Let $Q^{(\eta)}$ be defined by $Q_x^{(\eta)}:=(T_\eta)_\# Q_x$. Then $Q^{(\eta)}$ is finitely supported and, for every signed weight $w:\mathcal{X}\to[-1,1]$,
\[
\Err_P(Q^{(\eta)};w)\le \Err_P^{\mathrm{TV}}(Q;w)+\eta
\quad\text{and}\quad
\Delta_P(Q^{(\eta)})\le \Delta_P^{\mathrm{TV}}(Q)+\eta.
\]
Consequently, for every family $G$ of binary groups, $\MC_P(Q^{(\eta)};G)\le \MC_P^{\mathrm{TV}}(Q;G)+\eta$.
\end{proposition}

\begin{proof}
Choose any finite-grid rounding map $T_\eta$ with mesh at most $\eta$, for example nearest-grid rounding onto $\{0,\eta,2\eta,\dots,\lfloor 1/\eta\rfloor\eta,1\}$. For each grid value $z$ in the finite image of $T_\eta$, let $C_z:=T_\eta^{-1}(\{z\})$; these sets partition $[0,1]$, so $Q^{(\eta)}$ is finitely supported.

Fix a signed weight $w$. For each such $z$,
\begin{align*}
B_P(Q^{(\eta)};z,w)
&=
\E\!\left[w(X)\int_{C_z} (T_\eta(v)-Y)\,Q_X(dv)\right] \\
&=
\nu_{P,Q;w}(C_z)
+
\E\!\left[w(X)\int_{C_z} (T_\eta(v)-v)\,Q_X(dv)\right].
\end{align*}
Since $\{C_z\}_z$ is a finite measurable partition of $[0,1]$, we have $\sum_z |\nu_{P,Q;w}(C_z)|\le \|\nu_{P,Q;w}\|_{\mathrm{TV}}$. Therefore,
\begin{align*}
\Err_P(Q^{(\eta)};w)
&=
\sum_z |B_P(Q^{(\eta)};z,w)| \\
&\le
\sum_z |\nu_{P,Q;w}(C_z)|
+
\E\!\left[|w(X)|\int |T_\eta(v)-v|\,Q_X(dv)\right] \\
&\le
\|\nu_{P,Q;w}\|_{\mathrm{TV}}+\eta \\
&=
\Err_P^{\mathrm{TV}}(Q;w)+\eta.
\end{align*}
Taking the maximum over $g\in G$ gives the multicalibration claim.

For the prediction error,
\begin{align*}
\Delta_P(Q^{(\eta)})
&=
\E\!\left[\int |T_\eta(v)-\mu(X)|\,Q_X(dv)\right] \\
&\le
\E\!\left[\int |v-\mu(X)|\,Q_X(dv)\right]
+
\E\!\left[\int |T_\eta(v)-v|\,Q_X(dv)\right] \\
&\le
\Delta_P^{\mathrm{TV}}(Q)+\eta.
\end{align*}
\end{proof}

\begin{corollary}[Main lower bound extends to arbitrary randomized predictors up to constant-factor slack]
\label{cor:quantized-main}
Under the same construction and choice of constants as in Theorem~\ref{thm:lower-regular}, suppose a learning algorithm receives $n$ i.i.d.\ samples from $P_\theta$ and outputs an arbitrary randomized predictor $Q_S=(Q_{S,1},\dots,Q_{S,m})$, where each $Q_{S,i}$ is a probability measure on $[0,1]$.
Assume that for every $\theta\in\Theta_m$,
\[
\Pbb_{S\sim P_\theta^n,\ \mathcal{A}}
\!\left(
\MC_{P_\theta}^{\mathrm{TV}}(Q_S;G_m)\le \varepsilon
\right)
\ge \frac23.
\]
If
\[
\varepsilon \le \frac{c_3}{2}\,\frac{\gamma}{1+\log_2 m},
\]
then necessarily
\[
n\ge c_2\,\frac{m}{\gamma^2}.
\]
In particular, along the sequence $\gamma\asymp 1/m$, arbitrary randomized predictors also require
\[
n=\widetilde{\Omega}(\varepsilon^{-3})
\]
for this ECE notion of multicalibration.
\end{corollary}

\begin{proof}
Fix $\eta:=\varepsilon$. Given the learner's output $Q_S$, apply Proposition~\ref{prop:quantize-arbitrary} to obtain a finitely supported predictor $Q_S^{(\eta)}$ with
\[
\MC_{P_\theta}(Q_S^{(\eta)};G_m)
\le
\MC_{P_\theta}^{\mathrm{TV}}(Q_S;G_m)+\eta
\le
2\varepsilon
\le
c_3\,\frac{\gamma}{1+\log_2 m}.
\]
Thus the post-processed learner satisfies the hypothesis of Theorem~\ref{thm:lower-regular}, which gives $n\ge c_2\,m/\gamma^2$. Taking $\gamma\asymp 1/m$ gives the same $\widetilde{\Omega}(\varepsilon^{-3})$ reformulation as before.
\end{proof}

\end{document}